\documentclass[conference]{IEEEtran}

\usepackage{amsmath,amssymb,amsthm,mathtools}
\usepackage{enumitem}
\usepackage{xspace}
\usepackage[hidelinks]{hyperref}

\newcommand{\papertitle}{Transcript-Managed Transformers: Monotone Multi-Agent Collapse and Universality with Two Pop-Enabled Transcripts}
\newcommand{\paperauthor}{Sergey Salishev}
\newcommand{\paperkeywords}{Artificial intelligence, automata theory, finite-state machines, formal languages, multi-agent systems, transformers, Turing machines}

\hypersetup{
  pdftitle={\papertitle},
  pdfauthor={\paperauthor},
  pdfsubject={Computational expressivity of finite-precision transcript-managed transformers},
  pdfkeywords={\paperkeywords},
  pdfdisplaydoctitle=true
}

\newtheorem{theorem}{Theorem}
\newtheorem{corollary}{Corollary}
\newtheorem{proposition}{Proposition}
\newtheorem{lemma}{Lemma}
\theoremstyle{definition}
\newtheorem{definition}{Definition}
\newtheorem{remark}{Remark}

\newcommand{\TMT}{\mathsf{TMT}\xspace}
\newcommand{\TMTn}[1]{\mathsf{TMT}_{#1}}
\newcommand{\RTMT}{\mathsf{RTMT}\xspace}
\newcommand{\RTMTn}[1]{\mathsf{RTMT}_{#1}}
\newcommand{\OTMn}[1]{\mathsf{OTM}_{#1}}
\newcommand{\MONn}[1]{\mathsf{MON}_{#1}}
\newcommand{\FST}{\mathsf{FST}\xspace}
\newcommand{\DFA}{\mathsf{DFA}\xspace}
\newcommand{\DPDA}{\mathsf{DPDA}\xspace}
\newcommand{\DCFL}{\mathrm{DCFL}\xspace}
\newcommand{\CFL}{\mathrm{CFL}\xspace}
\newcommand{\RE}{\mathrm{RE}\xspace}
\newcommand{\F}{\mathbb{F}}
\newcommand{\Astk}{\mathcal{A}_{\mathrm{stk}}}
\newcommand{\Aapp}{\mathcal{A}_{\mathrm{app}}}
\newcommand{\Ck}{\mathcal{C}_{k}}
\newcommand{\PopContext}{\textbf{PopContext}}
\newcommand{\Trans}{\operatorname{Trans}}
\newcommand{\Acc}{\operatorname{Acc}}

\begin{document}

\title{\papertitle}

\author{%
\IEEEauthorblockN{\paperauthor}
\IEEEauthorblockA{%
\textit{AI Foundry}\\
San Francisco, CA, USA\\
\href{mailto:salishev@aifoundry.org}{\texttt{salishev@aifoundry.org}}}}

\maketitle

\begin{abstract}
We study transcript management for fixed, finite-precision causal Transformers.
A transcript is partitioned into channels of bounded blocks.  Each transition
consults a fixed visible suffix and may append one block, leaving the model,
weights, and token protocol unchanged.  The operation
$P_c:=\PopContext(c)$ deletes the newest block on channel $c$ and exposes its
predecessor.

We model the layer by the Transcript-Managed Transducer $\TMTn{k}$: one finite
controller, $k$ channels, and per-round actions from stay, push, and pop under a
caller-driven status map.  Fixed visible windows encode as finite symbols.  The
pop-free Restricted Transcript-Managed Transducer $\RTMTn{k}$ is the standard
append-only layer and, for every fixed $k$, realizes exactly the deterministic
finite-state transductions.  The same holds for every fixed finite agent
population under a monotone protocol that appends, routes, and copies visible
blocks.

Admitting $\{P_c\}_{c=1}^k$ restores pop.  Newest-first, a pop-enabled channel is
a stack; compiling to the Hopcroft--Ullman presentation transfers the classical
hierarchy: $\DCFL$ for $k=1$ and $\RE$ for every $k\ge2$.  Orchestrated
one-channel agents match one controller with $k$ channels, so two pop-enabled
transcripts---in one agent or two---suffice for universality.  Simulation costs
and invariance to fixed block size and visible radius are stated.  The bounds
fix precision, alphabets, blocks, visibility, controller state, and population;
growing exact context, hidden-block access, writable stores, and unbounded
\textbf{Spawn} add further state.
\end{abstract}


\section{Introduction}

Transformers are the standard architecture of modern AI~\cite{VaswaniEtAl2017,RadfordEtAl2019GPT2,BrownEtAl2020GPT3},
but deployed systems rarely use them as
isolated fixed-window sequence predictors.  Because consulting longer contexts increases memory
and inference costs, systems deliberately manage transcripts during execution.  Transcripts are
split, copied, delegated, resumed, compressed, retrieved, and occasionally deleted.
Reasoning scaffolds, tool loops, and multi-agent orchestration all rest on such transcript
management~\cite{YaoEtAl2023ReAct,WuEtAl2024AutoGen,HouEtAl2026AgenticLoops}.
This motivates a constructive question:

\begin{quote}
Which local transcript operation completes the abstract bounded-transcript model to a
universal machine, and can its semantics be added to a standard Transformer
without changing the token protocol or additional memory storage?
\end{quote}

Here a transcript is the sequence of bounded token blocks retained by the transcript manager
between model calls.  Each call receives the current input item and a fixed visible
suffix of each transcript channel.  The manager applies the returned transcript action.
The new transcript operation $P_c:=\PopContext(c)$ removes the newest block on channel $c$, so the
preceding block enters the next visible suffix.  Write $P:=P_1$.

Section~\ref{sec:model} defines the Transcript-Managed Transducer $\TMTn{k}$, the
normal form for this management layer.  One finite controller drives $k$ transcript
channels.  Each round consults the visible window of every channel together with the
current symbol, performs one action per channel from stay, push, and pop, and emits at
most one output symbol.  A status map marks each control state as an input round, an internal
round, or a halt, so the caller-supplied schedule is part of the model.  Write
$\TMT:=\TMTn{1}$.  The Restricted Transcript-Managed Transducer $\RTMTn{k}$ is the
pop-free case: every action is stay or push, which is the append-only layer of a standard
deployment.  Admitting $P_c$ on every channel lifts that restriction and returns
$\TMTn{k}$.

The channels are the parts of the one transcript that the deployment already keeps, and
$k$ counts how many of them admit pop.  A member of $\RTMTn{k}$ and a member of
$\TMTn{k}$ can share a controller, a block alphabet, a visible radius, and a forward
map; the second has one more available action.  A pop deletes a block and exposes the
block beneath it, so the retained transcript shortens.  Read newest block first, a
pop-enabled channel behaves as a stack.
Lemma~\ref{lem:classical-stack-normal-form} compiles the normal form above into
the Hopcroft--Ullman top-replacement presentation and back, with a bound on the steps per
classical transition; that compilation is what lets the classical one- and two-stack
classifications speak about $\TMTn{k}$, and two pop-enabled channels give universality.
Theorems~\ref{thm:exact-characterization}--\ref{thm:agent-stack-equivalence} are proved
directly in the normal form fixed here.

The abstract theorems classify the transcript manager.  Any fixed choice of attention,
depth, residual layers, and normalization supplies one finite local forward map under
the stated precision and visibility bounds.

We work under deployment constraints that stay fixed throughout:

\begin{enumerate}[leftmargin=1.5em]
\item finite precision and finite alphabets,
\item bounded transcript blocks and bounded exact visibility,
\item a fixed finite agent population with finite-state orchestration over agent transcripts,
\item the standard fixed, nontrainable execution schedule.
\end{enumerate}

\paragraph{Hierarchy summary}
Under these assumptions, write $\Trans(\mathcal C)$ for the input/output
transductions realized by a machine class $\mathcal C$,
$\Acc(\mathcal C)$ for its accepted languages,
$\MONn{m}$ for the class of monotone orchestrations with exactly $m$ agents, and
$\OTMn{k}$ for the class of Orchestrated Transcript Machines: $k$ one-channel agents under
one finite-state orchestration controller that owns their transcripts
(Definition~\ref{def:otm-family}).
For fixed $k$ and $m$, the transduction and acceptance classification is
\[
\begin{aligned}
\Trans(\RTMTn{k})
&=\Trans(\FST),\\
\Trans(\MONn{m})
&=\Trans(\FST),\\
\Trans(\OTMn{k})
&=\Trans(\TMTn{k}),\\
\Acc(\TMTn{1})
&=\DCFL,\\
\Acc(\TMTn{k})
&=\RE &&(k\ge2).
\end{aligned}
\]
We prove the three transduction identities.  The two acceptance identities are the
classical one- and two-stack classifications, transferred to this normal form by
Lemma~\ref{lem:classical-stack-normal-form} and recorded in
Corollary~\ref{cor:classical-acceptor-hierarchy} and
Theorem~\ref{thm:universal-augmentation-acceptance}.

\paragraph{Results and contributions}
This paper proves the fixed-population monotone collapse and the tagged-transcript normal
form with its correspondence to bounded-transcript controllers, including the equivalence
between $k$ orchestrated one-channel agents and one controller with $k$ channels together
with the simulation costs.
The one- and two-stack acceptance classifications are classical, and
Lemma~\ref{lem:classical-stack-normal-form} carries them into the normal form used here.

\begin{enumerate}[leftmargin=1.5em]
\item \emph{Monotone collapse.}
Every fixed finite collection of agents under a standard monotone protocol that may append,
route, and copy currently visible blocks realizes exactly the finite-state
transductions.  The theorem permits arbitrary finite controller logic and any sequence
of the allowed local and transcript operations.  Their joint visible configuration
forms a finite global summary that determines every future transition.  The same
summary also covers finite-branching nondeterministic acceptance and finite-precision
sampling under positive-probability acceptance.  The consequence for deployment is
concrete: with a fixed visible window, an append-and-copy agent population cannot parse a
general programming language, however large the retained transcript grows.
This limitation is visible in practice, through external grammars and symbolic checkers in
program synthesis~\cite{AlurEtAl2013,KobayashiEtAl2025}, context-free generalization
failures without stack memory~\cite{DeletangEtAl2023}, and unbounded repetition in
deployed agents~\cite{HouEtAl2026AgenticLoops}; our theorem states it directly for every
fixed population.

\item \emph{Transcript and agent correspondence.}
We show how a bounded channel-tagged physical transcript instantiates $\RTMTn{k}$ and
how a fixed-precision Transformer supplies its local transition map.  A separate
theorem equates the transductions of $k$ orchestrated one-channel agents with those of
one controller using $k$ pop-enabled channels.  An orchestrator that pops the transcripts
of two agents therefore has the power of one agent with two pop-enabled channels, which
covers both deployed layouts with one classification.

\item \emph{Exact hierarchy and invariance to fixed bounds.}
Every fixed visible radius can be encoded in one symbol from a finite enlarged
alphabet (Proposition~\ref{prop:bounded-transcript-correspondence}).  One
$\TMTn{k}$ round costs at most $k$ orchestration steps, one Turing-machine step costs
$O(1)$ two-agent steps, and $t$ stack operations use $O(t)$ transcript blocks
(Proposition~\ref{prop:simulation-overhead}).  One pop-enabled channel gives the
deterministic context-free languages; every fixed $k\ge2$ gives the recursively
enumerable languages, and additional channels may remain unused.  All effective
universal augmentations accept this same language class.
\end{enumerate}

The intermediate $\DPDA$ level is the first practical step beyond finite-state behavior
once exact pop is available: deterministic parsing~\cite{Knuth1965} and structured
instruction languages~\cite{Kuhn2014,VeizagaEtAl2021}.

\paragraph{Scope}
The upper bounds apply when the exact state consulted by one transition ranges over
a finite set.  Bounded windows, finite-state compaction, and bounded exact transcript
visibility satisfy this condition.
If $\F$ is the finite floating-point set, then every positional map
\[
p:\mathbb{N}\to \F^d
\]
has image
\[
\operatorname{im}(p):=\{p(t):t\in\mathbb N\}\subseteq\F^d,
\]
the set of positional vectors produced by $p$, so
\[
|\operatorname{im}(p)| \le |\F|^d < \infty.
\]
The bound comes from the precision of $\F$ alone, and no hypothesis on the rule generating
$p$ is required.  Because the domain $\mathbb N$ is infinite while $\operatorname{im}(p)$
is not, $p$ must repeat addresses, so position by itself supplies no unbounded consulted
state however far the index runs.  Under exact real or unbounded-precision arithmetic the
same map may take infinitely many values and this bound fails.
Models that later consult an exact unbounded sequence, an unbounded writable store, or
an unbounded population of spawned agents have a growing consulted-state space.
Standard read-only RAG into a bounded visible window supplies bounded external input
\cite{LewisEtAl2020RAG,OrenEtAl2024}.
The same finite-state assumption appears in formal analyses of restricted
transformers~\cite{Hahn2020,MerrillSabharwalSmith2022,MerrillSabharwal2023} and in classical
memory-bounded computation~\cite{StearnsHartmanisLewis1965}; see
Section~\ref{sec:related-work}.

\section{Related Work}
\label{sec:related-work}

The computational power of attention-based architectures has been studied from several
directions.  At the architectural level, the vanilla Transformer of
Vaswani et al.~\cite{VaswaniEtAl2017} does not prescribe how a deployment must manage
context size beyond attention.  Extensions with unbounded retained context can therefore
look like natural continuations of the same architecture even though they introduce a
different computational resource.

The following literature provides architectural upper bounds and broader context for our
contribution.

\paragraph{Transformer upper bounds}
Here $n$ denotes input length.
Hahn's fixed-layer restricted self-attention is not universal: it cannot recognize even
some regular languages~\cite{Hahn2020}.  In the same hard-attention regime, Hao, Angluin,
and Frank place unique-hard attention within $\mathsf{AC}^0$, the constant-depth
polynomial-size circuits of unbounded-fan-in Boolean gates, while averaging-hard attention
already recognizes the non-$\mathsf{AC}^0$ languages \textsc{Majority} and
\textsc{Dyck}-1~\cite{HaoAngluinFrank2022}.  Since $\mathsf{AC}^0$ excludes parity, the
unique-hard bound is strictly stronger than the threshold-circuit bounds that follow.
The saturated-transformer result is a nonuniform
constant-depth threshold-circuit upper bound and classifies a circuit family
\cite{MerrillSabharwalSmith2022}.  For $O(\log n)$-bit arithmetic,
logspace-uniform threshold-circuit simulations imply, under the standard conjecture that
deterministic logarithmic space is weaker than deterministic polynomial time, that one forward
pass misses some polynomial-time problems and therefore is not universal~\cite{MerrillSabharwal2023}.
Strobl et al.\ survey these and other formal-language results, emphasizing that their
apparently conflicting conclusions depend on precision, positional encodings, attention
variants, and uniformity assumptions~\cite{StroblEtAl2024Survey}.

\paragraph{Neural formal-language learning}
Bhattamishra et al.\ give Transformer constructions for a restricted family of counter
languages and empirically evaluate formal-language recognition, including
$a^n b^n c^n$~\cite{BhattamishraEtAl2020}.  Their positive finite-range result
uses the entire growing prefix and normalized prefix counts; their stated
finite-precision guarantee for this construction is bounded by the available bits.
Their construction uses growing consulted context.  Our classification fixes the
visible context.  Their finite-range result also supports the empirical possibility of
learning a two-counter procedure.
Del\'etang et al.\ train RNNs, LSTMs, Transformers, and stack- or tape-augmented
RNNs on transduction tasks spanning the Chomsky hierarchy~\cite{DeletangEtAl2023}.
Their stack and tape action policies are learned end-to-end, but both memory modules
are attached to RNN controllers; their evaluated Transformer is unaugmented.
That baseline enters Table~\ref{tab:stack-compare} as the row retaining no memory at all:
Theorem~\ref{thm:exact-characterization} fixes it at finite-state, and the deterministic
context-free tasks it misses are exactly those one pop-enabled channel supplies.
Length generalization reported in this literature is partial, and two separate conditions
govern it: the architecture must be able to represent the required memory discipline, and
training must actually find that representation.  Bhattamishra et al.\ show that gradient
descent biases Transformers toward simple, low-sensitivity functions, so the second
condition can fail even where the first holds~\cite{BhattamishraSimplicityBias2023}.
Our expressivity results speak only to the first.

\paragraph{Relation to classical stack models}
Classical memory bounds already induce strict computational hierarchies
\cite{StearnsHartmanisLewis1965}.  At the language-acceptance level, one stack gives
$\DCFL$ and two stacks give $\RE$; equivalently, the two-stack machine model is Turing
complete~\cite{Minsky1967,HopcroftUllman}.
We reach these results through a transducer normal form built for the deployed Transformer interface.
Lemma~\ref{lem:classical-stack-normal-form} compiles that form into the
Hopcroft--Ullman top-replacement presentation and back with step bounds, which is where
the classical acceptance results enter.  Our transduction theorems are proved in the
new normal form itself, and we identify PopContext with the pop action omitted by the
append-only management layer.
Our abstract results establish the tagged-transcript correspondence and the monotone
multi-agent collapse under fixed finite consulted state
(Definition~\ref{def:finite-consulted-state}).
PopContext is not the only conceivable memory primitive, and related classical intermediate variants
delimit the space of alternatives: pushdown transducers, visibly pushdown
automata~\cite{AlurMadhusudan2004}, counter machines~\cite{Minsky1967}, and ordered
multi-pushdown systems~\cite{Atig2012MultiPushdown}.  The maximality discussion in
Section~\ref{sec:main-results} is accordingly stated at the language-acceptance level for
effective universal augmentations, not as a claim that PopContext is the unique smallest
extension increasing the language acceptance power of the Transformer.

\paragraph{Post-training automaton extraction}
Weiss et al.\ take a fixed, already-trained RNN acceptor and use it as the membership
oracle for Angluin's $L^*$ algorithm~\cite{WeissGoldbergYahav2018}.  A finite
abstraction of the RNN is compared with each candidate deterministic finite automaton
($\DFA$); a disagreement yields either a counterexample to the $\DFA$ or a refinement of
the abstraction after consulting the RNN.  The result is a $\DFA$ extracted after neural
training.
Its relevance here is limited to the pattern of extracting an automaton after
neural training.  At the acceptance level, a $\DFA$ is a $\DPDA$ that never uses its
stack, and the inclusion is strict~\cite{HopcroftUllman}.  An $L^*$-style extraction
therefore reaches exactly the append-only level of
Theorem~\ref{thm:exact-characterization} and supplies no method for extracting a $\DPDA$
from a general $\TMTn{1}$ acceptor.

The remaining items are organized by which resource is allowed to grow and whether that
growth yields universality. To keep heterogeneous results comparable, we read each cited model along the 
axes stressed by the survey of Strobl et al.~\cite{StroblEtAl2024Survey}: numerical precision; whether
memory grows with input or runtime; whether that memory is writable and exactly accessible;
whether the control policy is learned or supplied; and whether the stated result concerns recognition, 
transduction, or finite-range learning. Table~\ref{tab:stack-compare} summarizes the comparison.
Its \emph{Task} column gives the generic setting a work evaluates, and \emph{Growing
Memory} names the resource the mechanism allows to grow.  \emph{Access} describes the
deployed read and update: soft when they are continuous or latent, exact when push and pop
are hard, and lossy when retained content is replaced by a summary; a qualifier such as
append-only records an operation the mechanism omits, and a dash marks a mechanism with no
such memory to access.
\emph{Representative power} is the language class of that growing
memory taken in its idealized exact form, so it states what the mechanism can express
rather than what the cited work claims or demonstrates. Turing completeness is recorded as $\RE$.
The last two columns therefore say together whether growth buys anything.
This reading makes mechanisms comparable across works whose evidence is otherwise
incommensurable.
Here $\DCFL\subsetneq\CFL$ separates the deterministic from the nondeterministic stack
relaxations, and a second exact unbounded store lifts either to $\RE$.  Every row keeps a
finite controller, so the comparison stays within deployable architectures rather than the
unbounded-precision constructions of \cite{PerezBarceloMarinkovic2021}.

\begin{table*}[t]
\caption{Stack- and memory-augmented neural models, grouped by controller.  Columns are
defined in the text; power is that of the growing memory in its idealized exact form, not
a claim of the cited work.}
\label{tab:stack-compare}
\centering
\renewcommand{\arraystretch}{1.25}
\begin{tabular}{@{}p{2.6cm}p{2.5cm}p{1.7cm}p{2.8cm}p{2.2cm}p{3.4cm}@{}}
\hline
Work & Task & Controller & Growing Memory & Access & Representative power\\
\hline
Grefenstette et al.~\cite{GrefenstetteEtAl2015}
  & Sequence transduction & RNN & Differentiable stack, queue, DeQue & Soft & $\DCFL$ (stack), $\RE$ (queue, DeQue)\\
Joulin and Mikolov~\cite{JoulinMikolov2015}
  & Algorithmic sequences & RNN & Differentiable stacks & Soft & $\DCFL$ (stack), $\RE$ (two stacks)\\
DuSell and Chiang~\cite{DuSellChiang2020}
  & Language recognition & RNN & Nondeterministic stack & Soft & $\CFL$\\
Chung and Siegelmann~\cite{ChungSiegelmann2021}
  & Machine simulation & RNN & Two unbounded stacks & Exact & $\RE$\\
\hline
Stack Attention~\cite{DuSellChiang2024}
  & Language recognition & Transformer & Stack inside attention & Soft & $\CFL$\\
Pushdown Layers~\cite{MurtyEtAl2023}
  & Language modeling & Transformer & Stack-like tape & Soft & $\DCFL$\\
StackTrans~\cite{ZhangStackTrans2025}
  & Chomsky hierarchy, reasoning & Transformer & Hidden-state stack & Soft & $\DCFL$\\
Oren et al.~\cite{OrenEtAl2024}
  & Long-context modeling & Transformer & Key-value cache & Exact, append-only & $\RE$\\
Transformer-XL~\cite{DaiEtAl2019TransformerXL}
  & Long-context modeling & Transformer & Cached segment states & Exact & Finite-state\\
Compressed memory~\cite{RaeEtAl2019Compressive,DaiEtAl2025}
  & Long-context modeling & Transformer & Context & Lossy & Finite-state\\
Unaugmented Transformer~\cite{DeletangEtAl2023}
  & Sequence transduction & Transformer & None & --- & Finite-state\\
\hline
This work
  & Sequence transduction & Any finite controller & Tagged transcript & Exact & $\DCFL$ ($k=1$), $\RE$ ($k\ge2$)\\
\hline
\end{tabular}
\end{table*}

\paragraph{Unbounded exact state}
P\'erez, Barcel\'o, and Marinkovic prove Turing completeness using arbitrary-precision
rational representations~\cite{PerezBarceloMarinkovic2021}.  This is outside fixed
precision.  Bhattamishra, Patel, and Goyal likewise analyze Transformer computational power
and universality under explicit model-specific
assumptions~\cite{BhattamishraPatelGoyal2020}.
The relevant classical caveat is that even two unbounded counters (equivalently,
two integer registers with suitable counter operations) already suffice for Turing
completeness~\cite{Minsky1967}; a fixed register count is finite-state only when register
values range over a fixed finite set.  Universal Transformers add recurrence and obtain
stronger capabilities under their own assumptions~\cite{DehghaniEtAl2018}.
Oren et al.\ cast decoder-only transformers as multi-state RNNs whose growing
key-value cache provides an unbounded exact retained sequence~\cite{OrenEtAl2024};
this holds in the idealized unbounded-length formulation, whereas a real finite-precision
deployment with a compressed or bounded cache remains finite-state in our formalization.
An unbounded exact sequence that each step may read and extend is a universal resource,
since writing successive machine configurations to it simulates a Turing machine step by
step, which is how the chain-of-thought simulations of
\cite{MerrillSabharwal2024CoT,LiWang2026EfficientTM} proceed; Table~\ref{tab:stack-compare}
records $\RE$ for that idealized cache and not as a claim of theirs.
Their own experiments point the same way: fixing the cache size yields a bounded
multi-state recurrent model that performs nearly on par with the full one.  The growing
memory recorded for them in Table~\ref{tab:stack-compare} is therefore the idealized
uncapped cache, not the one their experiments retain.

\paragraph{Differentiable stacks}
Differentiable stacks introduce another form of exact state in their idealized
temperature $\tau>0$ formulation: continuous block strengths range over exact real values.
Grefenstette et al.\ introduced continuous block strengths, newest-first fractional pop,
and unit-mass top reads, and report that the resulting models often recover the underlying
generating algorithm of a synthetic transduction grammar~\cite{GrefenstetteEtAl2015}.
That work instantiates the same machinery as a queue and as a double-ended queue; taken
exactly, either is Turing-equivalent, a queue by the standard queue-machine simulation of a
tape and a double-ended queue by exposing two stack ends at once, which is the second entry
recorded in its row.
Joulin and Mikolov proposed a distinct superposition stack with soft push, pop, and no-op,
and learn algorithmic sequences that require counting and
memorization~\cite{JoulinMikolov2015}.  They also run configurations holding several
stacks, and two exact unbounded stacks already suffice for Turing
completeness~\cite{Minsky1967,HopcroftUllman}, which is the second entry in their row.
DuSell and Chiang replaced the deterministic
relaxation by a weighting over nondeterministic stack runs, which reaches the context-free
class on language recognition~\cite{DuSellChiang2020}.  The deterministic variants provide deterministic pushdown power, hence $\DCFL$, and the
nondeterministic variants reach the full context-free class $\CFL$, which strictly
contains it.
The same relaxations were later carried into Transformer attention, where the
nondeterministic variant again represents $\CFL$~\cite{DuSellChiang2024}, and Chung and
Siegelmann extend their construction to
Grefenstette-style neural stacks~\cite{ChungSiegelmann2021}.
These are readings of the idealized exact structure, not claims made by the cited works;
each soft relaxation is evaluated on learning, and none of them establishes the class
recorded in Table~\ref{tab:stack-compare}.

\paragraph{Tape and stack-augmented RNN and Transformers}
With RNN controllers, the stack-augmented networks of Joulin and
Mikolov~\cite{JoulinMikolov2015} and of DuSell and Chiang~\cite{DuSellChiang2020} learn
algorithmic and context-free patterns, and the tape-augmented models evaluated above sit in
the same family~\cite{DeletangEtAl2023}.  Chung and Siegelmann prove Turing completeness for
a bounded-precision RNN with two dynamically growing stack
modules~\cite{ChungSiegelmann2021}; their universality uses two unbounded stores with exact
push and pop behavior rather than the relaxation.  With Transformer controllers, Stack
Attention places the store inside the attention operator and is evaluated on language
recognition~\cite{DuSellChiang2024}, Pushdown Layers maintain a stack-like tape synchronously
during autoregressive prediction to modulate attention, improving syntactic generalization in
language modeling~\cite{MurtyEtAl2023}, and StackTrans inserts a differentiable stack of
hidden states between Transformer layers rather than along the transcript, motivated
explicitly by the pushdown automaton as the minimal model for deterministic context-free
grammars and evaluated on Chomsky-hierarchy tasks as well as natural-language reasoning
benchmarks~\cite{ZhangStackTrans2025}.
Both Transformer-side stores are driven by soft push, pop, and no-op updates on a single
stack, so each idealizes to a deterministic pushdown store and is recorded at $\DCFL$.
StackTrans is instructive here because its deployed form returns to our finite-state regime
in two independent ways.  Its stack is allocated at a fixed size $S$, small in their
experiments, with overflow truncated to zero, which the authors describe as a form of
forgetting; and training parallelism is recovered by breaking the temporal dependency, so
that stack operations run over the layers of a single token rather than along the token
sequence.  Either restriction alone confines the reachable stack configurations to a fixed
finite set, so Definition~\ref{def:finite-consulted-state} is satisfied and
Theorem~\ref{thm:exact-characterization} applies.  The $\DCFL$ entry is therefore the
idealization the authors themselves state, an unbounded stack driven along the token
sequence, and not the model they train.
The evaluated task therefore matters for reading these rows: a language-class theorem and a
benchmark gain are not comparable evidence, so Table~\ref{tab:stack-compare} records the
task and the controller alongside the idealized power rather than collapsing them.
Our model differs from both lines in that the store is the tagged Transformer transcript
itself rather than a dedicated soft or latent memory module: each channel admits exact
push, stay, and pop in the $\TMTn{k}$ normal form.

\paragraph{Memory and retrieval}
Transformer-XL introduces segment-level recurrence over cached hidden
states~\cite{DaiEtAl2019TransformerXL}, whose span Rae and Razavi then vary to ask how much
of that memory the model actually needs~\cite{RaeRazavi2020}.  Separately, compressive
and embedding-based memories replace retained history by a finite
summary~\cite{RaeEtAl2019Compressive,DaiEtAl2025}.
In both, the retained context keeps growing as further history is folded in, but what a
step may consult of it does not: the cached span is fixed, and the summary has a fixed
number of slots that successive compressions overwrite.  Table~\ref{tab:stack-compare}
therefore lists a growing memory for these rows and still records them as finite-state.
How far the raw transcript grows does not affect the classification, because a standard
Transformer consults only a bounded visible window: the exact state entering one
transition ranges over a finite set whatever is retained behind it, and
Theorem~\ref{thm:exact-characterization} applies.  Compaction and compression therefore
change retention cost, not computational class.
Length generalization also depends on positional
encoding~\cite{PressEtAl2021ALiBi,KazemnejadEtAl2023PE}: under fixed floating-point arithmetic,
rotary encodings~\cite{SuEtAl2024RoFormer} are constrained by the RoPE
base~\cite{MenEtAl2024RoPEBase} and may lose positional distinctions when phases become
indistinguishable~\cite{Liu2026RoPEPhase}.  None of these works presents a Turing-universality
result.  With fixed precision and fixed cache, slot, and window bounds, their models remain
finite-state in our formalization.  Allowing retained activations or memory slots to grow
without bound and remain exactly consultable supplies growing consulted state.
A universality result additionally requires suitable state transitions.
The same resource distinction applies to retrieval.  Standard RAG retrieves a bounded chunk
from a read-only document store into the visible window.  Because
the store is read-only and the retrieved chunk is bounded, it supplies bounded external
input~\cite{LewisEtAl2020RAG}.  An exact store that the agent can update and
later query supplies persistent state and can support universal computation given suitable
read/write addressing.

\paragraph{Iteration and CoT}
Merrill and Sabharwal show that generalized pre-norm Transformers with a polynomially
bounded chain of thought recognize exactly the decision problems in
$\mathrm{P}$~\cite{MerrillSabharwal2024CoT}.  Their model lets each step consult a
growing transcript, uses logarithmic precision, and places a polynomial bound on the
chain-of-thought length.  Feng et al.\ show added effective depth for bounded written
derivations~\cite{FengEtAl2023CoT}.  Our theorem fixes the exact consulted state of
each step.  Under this condition, append-only CoT eventually cycles
(Corollary~\ref{cor:cot-cycle}).  This predicts eventual repetition of
tool-and-thought cycles observed in deployed agents~\cite{HouEtAl2026AgenticLoops}.

\paragraph{Context, agents, and external tools}
Li and Wang simulate space-$s(n)$ Turing machines with constant-bit Transformers using
an $O(s(n))$ context and $O(s(n)^c)$ chain-of-thought steps per simulated step, for
arbitrarily small fixed $c>0$~\cite{LiWang2026EfficientTM}.  Their exact transcript
grows with the simulated space.  Our monotone-collapse theorem fixes the exact consulted
state, so the two results describe different context regimes.
Rizvi-Martel et al.\ study
how agent count, bandwidth, and communication depth affect multi-agent reasoning
tasks~\cite{RizviMartelEtAl2025MultiAgent}.  Our theorem classifies every fixed finite
population governed by arbitrary finite controller logic and the stated local-call,
append, routing, and visible-copy operations.  Tool and function calls may supply
computation oracles, but their computational effect depends entirely on the interface: a
tool may be a finite-state transducer, a computable subroutine, a read-only database, a
writable store, or an undecidable
oracle~\cite{TiwariNalliDeshpande2026ToolOracles}.  Our bounds assume the finite-state and
read-only cases; writable store interface or undecidable tools add a separate computational resource.

\section{Model}
\label{sec:model}

\subsection{Transcript-Managed Transducer}

We first define the Transcript-Managed Transducer $\TMTn{k}$, the deterministic
transducer normal form used throughout the paper.  The deployed Transformer interface
shapes it: one round per model call, one transcript action per channel and at most one
output symbol per round, one written block per push, a status map that marks input rounds and internal
rounds, and a right-end marker with an explicit completed-run convention.  A pop on an
empty channel is a no-op, matching a transcript that stops at empty.
Lemma~\ref{lem:classical-stack-normal-form} compiles this form into the
Hopcroft--Ullman top-replacement presentation of the pushdown and multi-pushdown
models~\cite{HopcroftUllman,Atig2012MultiPushdown} and back, with a step bound in each
direction, so a result stated for either form applies to both.  The theorems of
Section~\ref{sec:main-results} work in the form fixed here; that compilation is what
carries the classical acceptance classifications across.
The symbols $\bot$ and $\dashv$ are distinguished internal-step and right-end markers,
respectively, outside every input alphabet.

\begin{definition}[Transcript-Managed Transducer $\TMTn{k}$]
\label{def:k-stack-transducer}
Define the stack-control alphabet
\[
\Astk
:=
\{\mathrm{stay},\mathrm{push},\mathrm{pop}\}.
\]
Fix $k\ge 1$.  A \emph{$k$-channel Transcript-Managed Transducer} is a tuple
\[
A := (Q_A,\Sigma,\Omega,\Lambda,\#,\delta_A,\chi_A,q_0),
\]
where $Q_A$ is a finite nonempty set of control states,
$q_0\in Q_A$ is the initial control state,
$\Sigma$ and $\Omega$ are finite input and output alphabets,
$\Lambda$ is a finite nonempty stack alphabet, and
$\# \notin \Lambda$ is the bottom marker.
The status map
\[
\chi_A:Q_A\to\{\mathrm{input},\mathrm{internal},\mathrm{halt}\}
\]
determines whether the next round consumes input or is internal, and marks halting
states.  It is a fixed component of the abstract transduction semantics.
Set
\[
\Sigma_\bot:=\Sigma\cup\{\dashv,\bot\},
\qquad
Q_A^\circ
:=
\chi_A^{-1}(\{\mathrm{input},\mathrm{internal}\}),
\]
and let
\[
\delta_A :
Q_A^\circ \times \Sigma_\bot \times (\Lambda \cup \{\#\})^k
\to
Q_A \times \Astk^k
\times\Lambda^k\times(\Omega\cup\{\varepsilon\}).
\]

On input $w\in\Sigma^*$, the initial configuration is
\[
(q_0,w\dashv,\#,\dots,\#).
\]
A configuration records a control state, the unread suffix of $w\dashv$, and $k$ stack
words in $\#\Lambda^*$.
If $\chi_A(q)=\mathrm{input}$, the round supplies and consumes the first unread symbol;
if $\chi_A(q)=\mathrm{internal}$, it supplies $\bot$ without consuming input.
The transducer applies $\delta_A$ to that symbol and the $k$ top symbols, performs the
returned actions simultaneously, and emits the returned output; $\varepsilon$ is silent.
The rightmost symbol of each stack word is its top.  Stay leaves that word unchanged,
push appends its returned write symbol, and pop removes the rightmost symbol when it lies
above $\#$.
If the action and write tuples are
$\mathbf a=(a_1,\dots,a_k)$ and $\boldsymbol\lambda=(\lambda_1,\dots,\lambda_k)$,
then action $\mathrm{push}$ on stack $c$ writes $\lambda_c$; the value $\lambda_c$ is
ignored for $\mathrm{stay}$ and $\mathrm{pop}$.  A pop on $\#$ leaves that stack
unchanged.

A run stops exactly on entering a state $q$ with $\chi_A(q)=\mathrm{halt}$.
The stopped run is \emph{complete} exactly when $\dashv$ has been consumed, so the unread
suffix is empty;
a premature halt, an input request with empty unread suffix, or an infinite run defines
no completed transduction.  The transducer's output on a completed run is the
concatenation of its nonsilent emissions.

For language recognition, partition the halting states into accepting and rejecting sets,
\[
\chi_A^{-1}(\mathrm{halt})
=
Q_A^{\mathrm{acc}}\mathbin{\dot\cup}Q_A^{\mathrm{rej}}.
\]
A completed run accepts or rejects according to the subset containing its final state.
This specialization is a $\TMTn{k}$ acceptor; it is a decider when every input completes
in one of the two subsets.
For fixed $k$, let $\TMTn{k}$ also denote the class of these instances, and write
$\TMT:=\TMTn{1}$.
Word $c$ of the configuration is transcript channel $c$, and its rightmost symbol is the
block currently exposed to the controller.  Read newest block first, channel $c$ behaves
as a stack over $\Lambda$, and for $k=1$ the acceptor is a deterministic pushdown
acceptor.  We use stack vocabulary for the three actions throughout, with the
understanding that each channel holds the blocks the deployment already retains.
Dropping the stack alphabets, stack arguments, and stack actions gives the ordinary
deterministic finite-state transducer under the same input, stopping, and completion
conventions; write $\FST$ for this class.
\end{definition}

\begin{definition}[Hopcroft--Ullman top-replacement presentation]
\label{def:hu-k-stack}
Take the Hopcroft--Ullman top-replacement transition form coordinatewise over $k$
stacks, as in the standard multi-pushdown model
\cite{HopcroftUllman,Atig2012MultiPushdown}.
A deterministic, output-bearing instance is
\[
H:=(Q,\Sigma,\Omega,\Lambda,\#,q_0,Q^{\mathrm{acc}},Q^{\mathrm{rej}},
\delta_{\mathrm{HU}}),
\]
where the finite sets and distinguished symbols have the meanings in
Definition~\ref{def:k-stack-transducer},
$Q^{\mathrm{acc}}\cap Q^{\mathrm{rej}}=\varnothing$, and, writing
$\Lambda_\#:={\Lambda\cup\{\#\}}$ and
$\overline\Sigma:=\Sigma\cup\{\dashv\}$,
\[
\begin{split}
\delta_{\mathrm{HU}}:\;&
\bigl(Q\setminus(Q^{\mathrm{acc}}\cup Q^{\mathrm{rej}})\bigr)
\times(\overline\Sigma\cup\{\varepsilon\})\times\Lambda_\#^k\\
&\rightharpoonup
Q\times(\Lambda_\#^*)^k\times(\Omega\cup\{\varepsilon\}).
\end{split}
\]
If
\[
\delta_{\mathrm{HU}}(q,a,\gamma_1,\dots,\gamma_k)
=
(q',u_1,\dots,u_k,o),
\]
then $a\in\overline\Sigma$ consumes $a$, whereas $a=\varepsilon$ consumes nothing,
and each stack $\beta_c\gamma_c$ becomes $\beta_cu_c$.
To preserve the unique bottom marker, require $u_c\in\Lambda^*$ when
$\gamma_c\ne\#$ and $u_c\in\#\Lambda^*$ when $\gamma_c=\#$.
Determinism means that the map is single-valued and, whenever its
$\varepsilon$-transition is defined for $(q,\gamma_1,\dots,\gamma_k)$, no
input-consuming transition is defined for the same state and top tuple.
The run starts from $(q_0,w\dashv,\#,\dots,\#)$, stops on entering
$Q^{\mathrm{acc}}\cup Q^{\mathrm{rej}}$, and is complete exactly when $\dashv$ has
been consumed.  A missing transition or an infinite run is incomplete; outputs and
acceptance follow Definition~\ref{def:k-stack-transducer}.
\end{definition}

\begin{lemma}[Hopcroft--Ullman multi-stack normal form]
\label{lem:classical-stack-normal-form}
For every fixed $k\ge1$, Definitions~\ref{def:k-stack-transducer}
and~\ref{def:hu-k-stack} realize exactly the same completed transductions and accept
exactly the same languages.  Let $L$ be the length of the longest replacement word in a
given $\delta_{\mathrm{HU}}$ table.  One transition of that table compiles into at most
$2+k(L+1)$ rounds of Definition~\ref{def:k-stack-transducer}, and one round of
Definition~\ref{def:k-stack-transducer} compiles into one transition of a
$\delta_{\mathrm{HU}}$ table with $L\le2$.
\end{lemma}

\begin{proof}
Let $L:=\max(\{0\}\cup\{|u_c|:\text{$u_c$ occurs in the transition table}\})$.
To compile $\delta_{\mathrm{HU}}$ into Definition~\ref{def:k-stack-transducer},
an internal dispatcher uses $(q,\gamma_1,\dots,\gamma_k)$ to select an
$\varepsilon$-transition when one exists and otherwise enters an input state.
After selecting the transition, finite control stores
$(q',u_1,\dots,u_k,o)$ and, in a fixed stack order, replaces each old top by one pop
followed by the symbols of $u_c$.  When $\gamma_c=\#$, write
$u_c=\#v_c$, retain $\#$, and push only the symbols of $v_c$.
At most $L$ pushes are needed per stack; emitting $o$ only on the final simulated step gives
at most $2+k(L+1)$ steps per classical transition.  Undefined transitions enter a fixed
nonhalting internal sink.

Conversely, one step of Definition~\ref{def:k-stack-transducer} is one
$\delta_{\mathrm{HU}}$ transition.  For current top $\gamma_c$, set
\[
u_c:=
\begin{cases}
\gamma_c, & a_c=\mathrm{stay},\\
\varepsilon, & a_c=\mathrm{pop}\ \text{and}\ \gamma_c\ne\#,\\
\#, & a_c=\mathrm{pop}\ \text{and}\ \gamma_c=\#,\\
\gamma_c\lambda_c, & a_c=\mathrm{push}.
\end{cases}
\]
Label it by the supplied input symbol when $\chi_A(q)=\mathrm{input}$ and by
$\varepsilon$ when $\chi_A(q)=\mathrm{internal}$, and retain the same next state and
output.  Halting states have no outgoing transition.  Both compilations preserve the
input suffix, stacks, state, and accumulated output at transition boundaries, which
proves the claim.
\end{proof}

The forward direction serializes an unbounded-width top replacement into rounds of
bounded width, which is what the one-action-per-model-call interface requires.  Through
this lemma the classical one- and two-stack acceptance results speak about the normal
form of Definition~\ref{def:k-stack-transducer}
(Corollary~\ref{cor:classical-acceptor-hierarchy} and
Theorem~\ref{thm:universal-augmentation-acceptance}).  The transduction theorems of
Section~\ref{sec:main-results} are proved in that normal form and use the lemma nowhere.

\subsection{Finite blocks and channels}

Fix a finite nonempty token alphabet $\Gamma$.
Fix a block bound $B \ge 1$.
Let
\[
\Delta := \Gamma^{[1,B]} := \{ w \in \Gamma^* : 1\le |w| \le B \}.
\]
Because $\Gamma$ is finite and $B$ is fixed, the block alphabet $\Delta$ is finite.
For $r\ge0$, write
\[
\Delta^{\le r}
:=
\{ s \in \Delta^* : |s| \le r \}
=
\{\varepsilon\}\cup\Delta^{[1,r]},
\]
where $\Delta^{[1,r]}:=\{s\in\Delta^*:1\le|s|\le r\}$ and $\Delta^{[1,0]}:=\emptyset$.

For each fixed integer $k \ge 1$, let
\[
\Ck := \{1,\dots,k\}
\qquad\text{and}\qquad
\Delta_k := \Ck \times \Delta.
\]
A \emph{$k$-channel transcript} is a word $H \in \Delta_k^*$.
Thus every bounded transcript block is explicitly tagged by its channel.
For $c \in \Ck$, write
\[
\pi_c(H) \in \Delta^*
\]
for the channel-$c$ projection obtained by retaining the payloads of precisely the blocks
tagged $c$, in their original order.
When $k=1$, we identify $(1,x)$ with $x$, so a one-channel transcript is simply a word in
$\Delta^*$.

\subsection{Restricted Transcript-Managed Transducer}

\begin{definition}[Finite consulted state]
\label{def:finite-consulted-state}
A system has \emph{finite consulted state} if its joint exact configuration consulted in
one transition ranges over a finite set.
Equivalently, its one-step transition factors through a finite consulted-state space.
\end{definition}

An unbounded sequence still has finite consulted state when every transition consults a
finite summary.  Exact random access, exact rescan, and unbounded persistent controller
state can create infinitely many distinguishable retained configurations.

\begin{definition}[Restricted Transcript-Managed Transducer $\RTMTn{k}$]
\label{def:tmt-family}
Define the append-only stack-control alphabet by
\[
\Aapp
:=
\{\mathrm{stay},\mathrm{push}\}.
\]
For fixed $k\ge1$, define the \emph{Restricted Transcript-Managed Transducer} class
$\RTMTn{k}$ as the class of instances
\[
A:=(Q_A,\Sigma,\Omega,\Lambda,\#,\delta_A,\chi_A,q_0)
\in\TMTn{k}
\]
whose action tuple returned by $\delta_A$ always lies in
$\Aapp^k$.
Thus pop is the only stack action excluded, and every
member of $\RTMTn{k}$ is one agent with one finite controller.
Stack $c$ is called transcript channel $c$, and $\Lambda$ is its finite abstract
transcript alphabet.
By convention,
\[
\RTMT:=\RTMTn{1}.
\]
Thus $\RTMTn{k}$ is the append-only management layer of a standard deployment: blocks
accumulate, and each call consults the visible window of every channel.
\end{definition}

\begin{definition}[PopContext]
\label{def:popcontext}
The primitive
\[
P_c:=\PopContext(c)
\]
is the action $\mathrm{pop}$ of Definition~\ref{def:k-stack-transducer} on channel $c$,
the one action omitted by $\RTMTn{k}$.  On a tagged transcript $H=u(c,x)v$ whose suffix
$v$ contains no block tagged $c$,
\[
P_c(H)=uv,
\]
and $P_c$ is a no-op when channel $c$ is empty.  For one channel, write
$P:=P_1=\PopContext$, so $\RTMTn{1}+P$ abbreviates $\RTMTn{1}+\{P_1\}$.
Admitting $\{P_c\}_{c=1}^k$ restores the action range $\Astk^k$, hence
\[
\RTMTn{k}+\{P_c\}_{c=1}^k
=
\TMTn{k}.
\]
\end{definition}

\begin{remark}[What $k$ and $P_c$ supply]
\label{rem:what-k-counts}
The $k$ channels partition the blocks of one transcript, so $k$ is a property of the
caller's bookkeeping.  The power comes from re-exposing a block that append-only
visibility had already passed: once a buried block can return, the finite summary of
Theorem~\ref{thm:exact-characterization} no longer determines the future, and the channel
contents below the visible window become consultable memory.
\end{remark}

The abstract $\TMTn{1}$ normal form reads one current top symbol.  
The published GPT-2 and GPT-3 context sizes of $1024$ and $2048$
tokens state large fixed physical radii \cite{RadfordEtAl2019GPT2,BrownEtAl2020GPT3}. 
The recoding below
packs any fixed physical window into that symbol.  For a block alphabet $\Delta$ and
visible radius $r$, the resulting window alphabet contains
$\sum_{j=1}^{r}|\Delta|^j$ symbols.  Increasing a fixed $r$ enlarges the finite local
state space.  Every fixed value of $r$ has the same transduction power.  A growing
radius $r=r(n)$ gives a growing-consulted-state model. 

\begin{definition}[Abstract bounded-transcript controller]
\label{def:bounded-transcript-controller}
Fix $k,r\ge1$, finite input and output alphabets $\Sigma$ and $\Omega$, a finite
nonempty block alphabet $\Delta$, a finite nonempty control set $Q$, an initial state
$q_0\in Q$, and a status map
$\chi:Q\to\{\mathrm{input},\mathrm{internal},\mathrm{halt}\}$.
For a physical transcript $H\in\Delta_k^*$, define
\[
\mathrm{vis}^{(k)}_r(H)
:=
\bigl(\mathrm{vis}_r(\pi_1(H)),\dots,\mathrm{vis}_r(\pi_k(H))\bigr),
\]
where $\mathrm{vis}_r(\pi_c(H))$ is the longest suffix of channel $c$ containing at most
$r$ blocks.  For $s\in\Delta^{\le r}$ and $x\in\Delta\cup\{\varepsilon\}$, let
$\mathrm{trim}_r(sx)$ be the longest suffix of $sx$ containing at most $r$ blocks.
Set $Q^\circ:=\chi^{-1}(\{\mathrm{input},\mathrm{internal}\})$.
A \emph{bounded-transcript controller}, used below as the wrapper transducer, is the
tuple
\[
\mathcal W
:=
(Q,\Sigma,\Omega,\Delta,k,r,\chi,q_0,g),
\]
where
\[
\begin{split}
g:\;&Q^\circ\times\Sigma_\bot\times(\Delta^{\le r})^k\\
&\longrightarrow
Q\times(\Delta_k\cup\{\varepsilon\})
\times(\Omega\cup\{\varepsilon\}),
\end{split}
\]
is its step map.  Its configuration contains a physical transcript
$H\in\Delta_k^*$, and $g$ consults $\mathrm{vis}^{(k)}_r(H)$.
The last argument is an ordered tuple: its coordinate is the channel identity.
Thus $g$ may apply any fixed-width map, including attention after adding a
channel-ID vector to each coordinate.
Initially $H=\varepsilon$.  A returned $(c,b)\in\Delta_k$ appends that tagged block;
the returned $\varepsilon$ performs no append.  Input consumption, output, stopping, and
completion follow Definition~\ref{def:k-stack-transducer}.
The status and completed-run convention belong to this abstract wrapper, not to the
deployed decoder.
In the pop-enabled form, replace the transcript-action codomain by
\[
\Delta_k\cup\{P_1,\dots,P_k,\varepsilon\},
\]
so one abstract simulation step performs at most one append or pop.
\end{definition}

\begin{proposition}[Bounded-transcript correspondence]
\label{prop:bounded-transcript-correspondence}
Every bounded-transcript controller induces a member of $\RTMTn{k}$.
Conversely, every member of $\RTMTn{k}$ has a bounded-transcript realization for some
finite physical block alphabet and $r=1$, using at most $k$ abstract simulation steps per
source transition, at most $k-1$ of them silent.
In the pop-enabled form, the correspondence identifies each physical $P_c$ with pop on
abstract stack $c$.
\end{proposition}

\begin{proof}
Define the finite window alphabet
\[
\widehat\Delta_r
:=
\Delta^{[1,r]}
=
\{s\in\Delta^*:1\le |s|\le r\}.
\]
For a channel word $b_1\cdots b_t\in\Delta^*$, set $\zeta_0:=\varepsilon$ and define
the window symbols
\[
\zeta_j:=\mathrm{vis}_r(b_1\cdots b_j)\in\widehat\Delta_r.
\]
Store $\zeta_1\cdots\zeta_t$ on the corresponding $\RTMTn{k}$ stack over
$\widehat\Delta_r$.  Its top is exactly the
window consulted by the transcript system.  Appending $b_{t+1}$ pushes
\[
\zeta_{t+1}:=\mathrm{trim}_r(\zeta_t b_{t+1}),
\]
with the empty channel handled by $\#$.  Hence every transcript step is one pop-free
stack-transducer step: an append tagged $c$ pushes $\zeta_{t+1}$ on stack $c$, and all other
stack actions stay.  The control, status, and output are unchanged.  A physical $P_c$
removes $\zeta_t$ and exposes $\zeta_{t-1}$, so the encoding commutes with abstract pop.

Conversely, for an $\RTMTn{k}$ with stack alphabet $\Lambda$, take the physical block
alphabet to be $\Delta':=\Lambda$, set $r=1$, and use one transcript channel for each
stack.
The finite controller stores the returned action, write, and output tuples.
If all actions stay, one simulation step emits the output and enters the next abstract
control state.  Otherwise, the first simulation step receives the symbol selected by the
abstract status, and the controller serializes the non-stay channel actions in increasing
channel order.  Every remaining step is internal and silent; only the final step emits
the pending output and enters the next abstract control state.  There are at most $k$
simulation steps and at most $k-1$ auxiliary steps.  In the pop-enabled form, push and
pop are serialized identically.  Induction at simulation boundaries preserves the
control state, channel contents, unread suffix, status, emitted output, stopping, and
completion.
This serialization establishes transduction equivalence.  The sample model applies one
hard or relaxed stack action inside each caller-supplied token position.
\end{proof}

Thus every fixed physical radius has an equivalent $r=1$ realization and carries no
additional theoretical power.

\begin{remark}[Transformer instantiation]
Consider the standard causally masked decoder stack of Transformer
blocks~\cite{VaswaniEtAl2017}, with token vocabulary $V$, finite-precision number
system $\F$, embedding dimension $d$, and fixed weights in $\F$.
A transcript-manager wrapper selects the finite visible suffix supplied to the standard
forward map.
Write the decoder instance as
\[
\begin{split}
\mathcal D
:=\bigl(&V,\F,d,N,L,(\mathcal H_\ell)_{\ell=0}^{L-1},e,p,\\
& (W_Q^{\ell,a},W_K^{\ell,a},W_V^{\ell,a})_
{\substack{0\le\ell<L\\a\in\mathcal H_\ell}},
(\operatorname{Block}_\ell)_{\ell=0}^{L-1},
W_{\mathrm{out}},b_{\mathrm{out}}\bigr),
\end{split}
\]
where $N$ is the maximum context length, $L$ is the number of layers,
$\mathcal H_\ell$ is the finite head set of layer $\ell$, and
$\operatorname{Block}_\ell$ contains that layer's remaining fixed Transformer
operations.
For each $a\in\mathcal H_\ell$, let $d_{\ell,a}\ge1$ be the head width.  The
coordinates in this tuple have types
\[
\begin{gathered}
e:V\to\F^d,\qquad
p:\{1,\dots,N\}\to\F^d,\\
W_Q^{\ell,a},W_K^{\ell,a},W_V^{\ell,a}
\in\F^{d_{\ell,a}\times d},\\
\operatorname{Block}_\ell:
\F^d\times
\F^{\sum_{a\in\mathcal H_\ell}d_{\ell,a}}
\to\F^d,\\
W_{\mathrm{out}}\in\F^{|V|\times d},
\qquad b_{\mathrm{out}}\in\F^{|V|}.
\end{gathered}
\]
Let
\[
F_{\mathcal D}:V^{[1,N]}\longrightarrow
\bigl(\F^{|V|}\bigr)^{[1,N]}
\]
be its standard forward map at maximum context length $N$, with
$F_{\mathcal D}(V^n)\subseteq(\F^{|V|})^n$ for every $1\le n\le N$.
Instantiate the wrapper transducer of
Definition~\ref{def:bounded-transcript-controller} with $k:=1$, so it exposes one
transcript suffix of at most $r$ blocks.  Fix a transcript tokenizer
\[
\operatorname{Tok}:\Sigma_\bot\times\Delta^{\le r}
\longrightarrow V^{[1,N]}
\]
and a finite state/action/output decoder
\[
\begin{split}
\operatorname{Dec}:\;&Q^\circ\times\Sigma_\bot
\times\bigl(\F^{|V|}\bigr)^{[1,N]}\\
&\longrightarrow
Q\times(\Delta\cup\{\varepsilon\})
\times(\Omega\cup\{\varepsilon\}).
\end{split}
\]
Under the standing $k=1$ identification $(1,x)\leftrightarrow x$, the transcript-action
codomain $\Delta\cup\{\varepsilon\}$ matches the general wrapper's
$\Delta_1\cup\{\varepsilon\}$.
Thus $r$ is the maximum number of exposed physical transcript blocks.  The parameter
$N$ is the maximum number of tokens accepted by $\mathcal D$;
$\operatorname{Tok}$ serializes the exposed blocks and current input symbol within
that token bound.
For $s\in\Delta^{\le r}$, define
\[
g_{\mathcal D,\operatorname{Tok},\operatorname{Dec}}(q,x,s)
:=
\operatorname{Dec}\bigl(
q,x,F_{\mathcal D}(\operatorname{Tok}(x,s))
\bigr).
\]
The resulting wrapped Transformer transducer is
\[
\mathcal W[\mathcal D,\operatorname{Tok},\operatorname{Dec}]
:=
\bigl(
Q,\Sigma,\Omega,\Delta,1,r,\chi,q_0,
g_{\mathcal D,\operatorname{Tok},\operatorname{Dec}}
\bigr).
\]
The objects $\mathcal D$, $\operatorname{Tok}$, and $\operatorname{Dec}$ parameterize
the local step map $g$.  The decoder tuple and its Transformer block equations stay
fixed.
Since $V$, $\F$, and $N$ are finite, the domain and range of
$F_{\mathcal D}$ are finite.  Hence the exact local transition
$g_{\mathcal D,\operatorname{Tok},\operatorname{Dec}}$ factors through a finite set of
visible contexts.

Take $\Sigma=\Omega=\Gamma=V$ and choose a block bound $B$ for transcript chunks.
In the one-channel token-level realization, $\Delta=V$, $B=1$, and $r=N$ when the
current input and boundary information are already present in the exposed token sequence.
If the input or boundary information is serialized separately, its tokens also count
toward $N$; in every case the exact requirement is
$|\operatorname{Tok}(x,s)|\le N$.
For chunked transcripts, $\operatorname{Tok}$ may use any fixed segmentation satisfying
the same bound.
Proposition~\ref{prop:bounded-transcript-correspondence} packs the entire fixed portion
consulted by one forward step into one finite abstract top symbol.
This construction specifies the local transcript interface.  The standard
autoregressive protocol supplies prompt processing, successive invocations, and
end-of-sequence or length stopping.  The wrapper uses this fixed schedule.
\end{remark}

\subsection{Monotone orchestration}

Throughout this subsection, each $\RTMT$ agent uses the $r=1$ bounded-transcript
realization from Proposition~\ref{prop:bounded-transcript-correspondence}.
Because the population is finite, a disjoint-union recoding gives all agents a common
finite block alphabet $\Delta$.
For $H\in\Delta^*$, let $\mathrm{top}_\#(H)$ be its final block, or $\#$ when
$H=\varepsilon$.

A \emph{monotone orchestration system} consists of finitely many such agents
$\mathcal{M}_1,\dots,\mathcal{M}_m$, physical transcripts
$H_1,\dots,H_m\in\Delta^*$, local control states $q_1,\dots,q_m$, and an
orchestration-controller state
$q_{\mathrm{ctl}}\in Q_{\mathrm{ctl}}$, where $Q_{\mathrm{ctl}}$ is finite.
Each agent starts in its designated initial state with an empty transcript, and the
controller fixes an initial state $q_{\mathrm{ctl},0}\in Q_{\mathrm{ctl}}$.
It also fixes a global status map
\[
\chi_{\mathrm{ctl}}:
Q_{\mathrm{ctl}}\to\{\mathrm{input},\mathrm{internal},\mathrm{halt}\}.
\]
It governs the shared input convention: the agents have no separate unread-input
cursors, and a selected local transition receives the symbol chosen from the single
global unread suffix.

The population is fixed and finite, as required by the collapse results below;
unbounded \textbf{Spawn} is treated separately.

\begin{definition}[Monotone protocol]
A \emph{monotone protocol step} is any deterministic controller action that depends only on
\begin{enumerate}[leftmargin=1.5em]
\item the current controller state,
\item the current symbol in $\Sigma_\bot$,
\item the current local agent states,
\item the exposed blocks $\mathrm{top}_\#(H_1),\dots,\mathrm{top}_\#(H_m)$,
\end{enumerate}
and which, apart from finite controller-state updates, performs at most one of the
following transcript actions:
\begin{enumerate}[label=(M\arabic*),leftmargin=1.8em]
\item execute one step of one agent's bounded-transcript realization whose
input/internal status agrees with $\chi_{\mathrm{ctl}}(q_{\mathrm{ctl}})$,
\item append one block from $\Delta$ to some transcript,
\item copy an exposed block to the end of another transcript.
\end{enumerate}
Each step may also emit at most one symbol from a fixed finite output alphabet.
No monotone protocol step may delete a block or otherwise reveal transcript content
below the exposed blocks.
On input $w$, the global unread suffix starts as $w\dashv$.
An input controller state consumes its first symbol, an internal controller state supplies
$\bot$, and a halting controller state admits no step.
The run is complete exactly when it halts after consuming $\dashv$, as in
Definition~\ref{def:k-stack-transducer}.
For language recognition, partition the halting controller states as
\[
\chi_{\mathrm{ctl}}^{-1}(\mathrm{halt})
=
Q_{\mathrm{ctl}}^{\mathrm{acc}}
\mathbin{\dot\cup}
Q_{\mathrm{ctl}}^{\mathrm{rej}}.
\]
\end{definition}

The visibility restriction is exact: a copy source is one of the currently exposed
blocks.  Addressing a fixed finite set of transcript positions remains finite-state
because their values fit in the finite consulted-state summary.  Addressing arbitrary
positions in an unbounded retained transcript creates infinitely many consultable
configurations.  Exact reload of discarded text and an unbounded writable store create
the same additional resource.

For each fixed $m\ge1$, let $\MONn{m}$ denote the class of monotone
orchestration systems with exactly $m$ agents.

\begin{definition}[Spawn]
The operation \textbf{Spawn} creates a fresh agent with its own local state and
transcript, initialized in a fixed start configuration.
If the controller may invoke \textbf{Spawn} unboundedly often, then the population of
live agents is no longer a priori finite.
Operationally, it models starting a fresh agent session with its own retained transcript.
\end{definition}

\begin{remark}[Deployed transcript primitives]
Table~\ref{tab:deployed-monotone} states when common deployment patterns remain monotone.
\PopContext, unbounded writable memory, and unbounded
\textbf{Spawn} are outside the monotone protocol defined here.
\end{remark}

\begin{table}[!t]
\caption{Deployment patterns and monotone conditions.}
\label{tab:deployed-monotone}
\centering
\small
\begin{tabular}{@{}p{0.27\linewidth}p{0.27\linewidth}p{0.36\linewidth}@{}}
\hline
Pattern & Ops / primitives & Monotone condition \\
\hline
CoT / ReAct scratchpad~\cite{MerrillSabharwal2024CoT,FengEtAl2023CoT,YaoEtAl2023ReAct}
& Append
& Each step reads a fixed-size exact context and cannot recover discarded tokens. \\
Standard RAG~\cite{LewisEtAl2020RAG}
& Copy / insert into window
& Document store is read-only. \\
Finite-summary compaction~\cite{RaeEtAl2019Compressive,DaiEtAl2025}
& Replace by finite summary
& Discarded exact text is never reloaded. \\
Multi-agent handoff~\cite{RizviMartelEtAl2025MultiAgent,WuEtAl2024AutoGen}
& Copy shared / visible messages
& Agent population stays fixed. \\
KV cache / durable store~\cite{OrenEtAl2024}
& Retain / drop cache entries
& Consulted state remains a fixed finite context; dropped exact entries are never
re-exposed. \\
\hline
\end{tabular}
\end{table}

\subsection{Orchestrated Transcript Machine}

\begin{definition}[Orchestrated Transcript Machine with $k$ agents]
\label{def:otm-family}
Fix $k\ge1$.  Let $A_i\in\RTMT$ range over one-channel agents, let
$\Sigma,\Omega$ range over finite input and output alphabets, and let
$Q_{\mathrm{ctl}}$ range over finite orchestration-control sets.
Use the $r=1$ realizations of
Proposition~\ref{prop:bounded-transcript-correspondence}; after finite recoding, let
$\Delta$ be their common physical block alphabet and $Q_i$ the local control of
agent $i$.  Set
\[
Q_{\mathcal O}
:=
Q_{\mathrm{ctl}}\times\prod_{i=1}^k Q_i,
\qquad
\Delta_\#:=\Delta\cup\{\#\}.
\]
Let
\[
\begin{aligned}
\mathcal A_{\mathcal O,k}
&:=
\{\mathrm{idle}\}\cup
\{\mathrm{step}(i):i\in\Ck\}\\
&\quad\cup
\{\mathrm{push}(i,b):
  i\in\Ck,\ b\in\Delta\}\\
&\quad\cup
\{\mathrm{route}(i,j):
  i,j\in\Ck,\ i\ne j\}\\
&\quad\cup
\{\mathrm{pop}(i):i\in\Ck\}.
\end{aligned}
\]

Define $\OTMn{k}$ as the class of instances
\[
\mathcal O
:=
\bigl(
(A_i)_{i=1}^k,Q_{\mathrm{ctl}},\Sigma,\Omega,\Delta,
\chi_{\mathcal O},q_{\mathcal O,0},\delta_{\mathcal O}
\bigr),
\]
where $q_{\mathcal O,0}\in Q_{\mathcal O}$ contains the designated controller and
agent initial states,
\[
\begin{aligned}
\chi_{\mathcal O}:
Q_{\mathcal O}
&\to\{\mathrm{input},\mathrm{internal},\mathrm{halt}\},\\
Q_{\mathcal O}^{\circ}
&:=
\chi_{\mathcal O}^{-1}(\{\mathrm{input},\mathrm{internal}\}),
\end{aligned}
\]
and
\[
\begin{split}
\delta_{\mathcal O}:\;&
Q_{\mathcal O}^{\circ}\times\Sigma_\bot\times\Delta_\#^k\\
&\longrightarrow
Q_{\mathcal O}\times\mathcal A_{\mathcal O,k}
\times(\Omega\cup\{\varepsilon\})
\end{split}
\]
is a deterministic sequential orchestration map.  It consults the finite global
control and the exposed top of each physical transcript $H_i\in\Delta^*$.
An \emph{ordinary call} to agent $A_i$ advances only that agent's local control
state under the shared input symbol selected by $\chi_{\mathcal O}$; it does not
consult or modify any other agent's state, and the orchestration retains exclusive
control of every transcript.  Action $\mathrm{step}(i)$ is one ordinary call with no
transcript change.  Action $\mathrm{push}(i,b)$ is one ordinary call together with an
orchestration-level append of the action-specified block $b\in\Delta$ to $H_i$; the
appended block is taken from the action, not from an agent-emitted return value.
Action $\mathrm{route}(i,j)$ appends the exposed final block of $H_i$ to $H_j$ and
leaves $H_i$ and all local states unchanged; it is a no-op when $H_i$ is empty.
Action $\mathrm{pop}(i)$ is issued by the orchestration layer and deletes the final
block of $H_i$ when present; idle changes no agent or transcript.
The initial configuration is
$(q_{\mathcal O,0},w\dashv,\varepsilon,\dots,\varepsilon)$.
Input consumption, output, halting, and completion follow
Definition~\ref{def:k-stack-transducer}.

Thus $\OTMn{k}$ has $k$ one-channel agents.  The class $\TMTn{k}$ has one agent with
$k$ tagged channels.
\end{definition}

\begin{remark}[Compaction, routing, and serialization]
Compaction is not an extra primitive: an agent produces a summary by an ordinary step,
and orchestration-level pop removes superseded blocks.  Cross-agent handoff is
$\mathrm{route}(i,j)$; a destructive move is route followed by $\mathrm{pop}(i)$.
The model records a sequential event trace.  Every deterministic parallel deployment
has a deterministic serialization of its completed agent events.  If several events
read one pre-round state, its finite control and exposed-top tuple can be retained in
$Q_{\mathrm{ctl}}$ while their effects are serialized.
\end{remark}

\section{Theoretical Results}
\label{sec:main-results}

All results in this section concern the abstract machine classes of
Section~\ref{sec:model}.

\subsection{Exact characterization of one agent}

\begin{theorem}[Exact single-agent characterization]
\label{thm:exact-characterization}
For every fixed $k\ge 1$,
\[
\Trans(\RTMTn{k})=\Trans(\FST).
\]
\end{theorem}

\begin{proof}
Let $A\in\RTMTn{k}$ have control set $Q_A$ and stack alphabet $\Lambda$.
The finite summary space
\[
S:=Q_A\times(\Lambda\cup\{\#\})^k
\]
records the control state and current top of every channel.  Stay preserves a top and
push replaces the recorded top with the written symbol.  Because pop is unavailable, no buried symbol
can become visible again.  Hence the next summary, status, and output depend only
on $S$, giving an equivalent deterministic finite-state transducer.

Conversely, a deterministic finite-state transducer under the same status convention is
an $\RTMTn{k}$ over any singleton stack alphabet whose action tuple is always
$\mathrm{stay}^k$.
Thus both classes realize the same completed transductions.
\end{proof}

In particular, every $\RTMTn{k}$ recognizes exactly the regular languages and is strictly
weaker than a pushdown automaton, which recognizes the nonregular language
$\{a^n b^n:n\ge1\}$.

\subsection{Monotone protocols stay finite-state}

\begin{theorem}[Monotone multi-agent collapse]
\label{thm:monotone-collapse}
For every fixed $m\ge1$,
\[
\Trans(\MONn{m})
=
\Trans(\FST).
\]
\end{theorem}

\begin{proof}
For each agent, Proposition~\ref{prop:bounded-transcript-correspondence} and
Theorem~\ref{thm:exact-characterization} give a finite local summary $S_i$ containing
its control and encoded visible suffix.  The global summary space
\[
S
:=
Q_{\mathrm{ctl}}
\times
\prod_{i=1}^m S_i
\]
is finite.  A local transition updates one $S_i$.  A direct append or exposed-block copy
updates the target's encoded visible suffix from the source and target suffixes; no permitted operation
reveals a buried symbol.  Thus every protocol step induces a deterministic map on $S$,
which is sufficient to determine an equivalent finite-state transition.
Conversely, every deterministic finite-state transducer is realized by a member of
$\MONn{m}$ whose agents always stay and whose orchestration controller implements
the finite-state transition.
\end{proof}

The fixed-population hypothesis is essential: after $n$ unbounded \textbf{Spawn}
operations the summary has $n$ local components, so no fixed finite product
$S$ covers all configurations.

\begin{remark}[Nondeterministic and probabilistic protocols]
Finite branching does not enlarge the summary space $S$.  Nondeterministic existential
acceptance reduces by the subset construction, while assigning positive probability to
each enabled branch gives the randomized-machine construction of
\cite[Sec.~11.4.3]{HopcroftUllman}: its random tape resolves a path through the same
finite summary.  Thus ordinary finite-precision token sampling is a probabilistic
finite-state machine on $S$ and accepts with positive probability exactly when some
corresponding nondeterministic branch accepts.
\end{remark}

\begin{corollary}[Append-only CoT cycles]
\label{cor:cot-cycle}
Under the hypotheses of Theorem~\ref{thm:monotone-collapse}, allowing unboundedly many
append-only internal steps after a finite input is consumed yields an ultimately periodic
run whenever it does not halt.
\end{corollary}

\begin{proof}
By Theorem~\ref{thm:monotone-collapse}, the global process factors through a finite state
set $S$.
After the input is consumed, its internal transition on $S$ is deterministic and
autonomous.
Hence an infinite run eventually revisits a state, after which both the states and emitted
blocks repeat periodically.
\end{proof}

\subsection{Multi-agent correspondence and stack hierarchy}

\begin{theorem}[Agent/channel equivalence]
\label{thm:agent-stack-equivalence}
For every fixed $k\ge1$,
\[
\Trans(\OTMn{k})=\Trans(\TMTn{k}).
\]
\end{theorem}

\begin{proof}
Let $\mathcal O\in\OTMn{k}$.  Use $Q_{\mathcal O}$ as finite controller state, $\Delta$
as channel alphabet, and transcript $H_i$ as channel $i$ above its bottom marker.
Actions $\mathrm{step}(i)$ and idle become the all-stay tuple,
$\mathrm{push}(i,b)$ becomes push of $b$ on coordinate $i$, and
$\mathrm{route}(i,j)$ becomes stay when stack $i$ is empty and otherwise push of its
top symbol on coordinate $j$.
Action $\mathrm{pop}(i)$ becomes pop on coordinate $i$; every other coordinate stays.
The transition domain, consulted top tuple, status, input convention, and output are
otherwise identical, yielding a step-for-step deterministic simulation by a member of
$\TMTn{k}$.

Conversely, let a member of $\TMTn{k}$ have channel alphabet $\Lambda$.
Take $k$ one-state $\RTMT$ agents with common physical block alphabet
$\Delta:=\Lambda$ whose finite input interface selects stay or push of a specified
$b\in\Lambda$.  Store the simulated control state and pending transition in
$Q_{\mathrm{ctl}}$.  At the start of each simulated transition,
$\delta_{\mathcal O}$ consults the same top tuple and serializes the at most $k$
non-stay coordinate actions in increasing channel order:
$\mathrm{push}(i,b)$ implements push,
$\mathrm{pop}(i)$ implements pop, and an all-stay transition needs one idle round.
The first round consumes the source symbol;
remaining rounds are internal $\bot$-steps.  Only the final round emits the pending
output and completes the simulated transition.  Induction over completed simulated transitions
preserves the simulated state, stacks, unread input, status, and emitted output.
\end{proof}

\begin{corollary}[Pop-enabled single-/multi-agent correspondence]
\label{cor:tmt-otm-equivalence}
For every fixed $k\ge1$,
\[
\Trans(\RTMTn{k}+\{P_c\}_{c=1}^k)=\Trans(\OTMn{k}).
\]
\end{corollary}

\begin{proof}
By Definition~\ref{def:popcontext},
$\RTMTn{k}+\{P_c\}_{c=1}^k=\TMTn{k}$.
Apply Theorem~\ref{thm:agent-stack-equivalence}.
\end{proof}

\begin{corollary}[Classical acceptor hierarchy]
\label{cor:classical-acceptor-hierarchy}
Under the acceptor specialization, both $\RTMTn{1}+P$ and $\OTMn{1}$ recognize exactly
the deterministic context-free languages.  For every fixed $k\ge2$, both
$\RTMTn{k}+\{P_c\}_{c=1}^k$ and $\OTMn{k}$ are
Turing-complete~\cite{Minsky1967,HopcroftUllman}.
\end{corollary}

\begin{proof}
Definition~\ref{def:popcontext} identifies
$\RTMTn{k}+\{P_c\}_{c=1}^k$ with $\TMTn{k}$, and
Theorem~\ref{thm:agent-stack-equivalence} identifies $\OTMn{k}$ with the same
transduction class.  The simulations in its proof preserve halting status, so the
identification also holds for the acceptor specializations.
Lemma~\ref{lem:classical-stack-normal-form} transfers the
classical Hopcroft--Ullman results to this normal form: one stack gives the
deterministic context-free languages, and two stacks simulate a Turing machine.
In the standard two-stack tape encoding, one stack stores the cells left of the head
with the nearest cell on top, and the other stores the current cell and the cells to its
right.  A head move writes the current symbol and transfers the boundary symbol between
the stacks; finite control stores the Turing-machine state and handles blanks.
In transcript terms, one tape cell is one block: a head move appends one block to one
channel and applies $P_c$ to the other, so the two retained transcripts hold the tape
split at the head.
For $k>2$, the additional channels may remain unused.
\end{proof}

Thus the acceptance hierarchy has two cases: $k=1$ gives $\DCFL$, and every fixed
$k\ge2$ gives $\RE$.  Two pop-enabled transcripts already reach $\RE$; additional
agents or channels preserve this language class.

For fixed $k$, an \emph{augmentation} $X$ of $\RTMTn{k}$ is an extension of its
configuration and transition interface by specified deterministic primitives;
$\RTMTn{k}+X$ denotes the resulting machine class.  Call $X$ \emph{effective} if every
extended configuration has a finite description and every added one-step primitive is
Turing-computable.  Call it \emph{Turing-complete} if $\RTMTn{k}+X$ simulates a
Turing machine.
For a fixed channel count $k$, write $X\preceq_{\mathrm{acc}}Y$ when every
language accepted by $\RTMTn{k}+X$ is accepted by $\RTMTn{k}+Y$.
Let $\RE$ denote the class of recursively enumerable languages.

\begin{theorem}[Acceptor maximality of PopContext for $k\ge2$]
\label{thm:universal-augmentation-acceptance}
For every $k\ge2$, $\RTMTn{k}+\{P_c\}_{c=1}^k$ accepts exactly $\RE$ and is
acceptor-equivalent to every effective Turing-complete augmentation.
Consequently, all effective universal augmentations form one equivalence class under
$\preceq_{\mathrm{acc}}$.
\end{theorem}

\begin{proof}
By Definition~\ref{def:popcontext},
$\RTMTn{k}+\{P_c\}_{c=1}^k=\TMTn{k}$.
By Lemma~\ref{lem:classical-stack-normal-form}, the classical two-stack construction
applies to this normal form and gives Turing completeness for $k\ge2$.
For the upper bound, a Turing machine encodes each finite configuration description and
computes successive transitions.
Thus the channelwise PopContext augmentation and every effective Turing-complete
augmentation $X$ accept exactly $\RE$~\cite{HopcroftUllman,Minsky1967}.
\end{proof}

\begin{remark}[Intermediate classes]
Theorem~\ref{thm:universal-augmentation-acceptance} begins its comparison at
universality.  Channelwise PopContext is one effective route to $\RE$.
Nonuniversal nonmonotone primitives may realize intermediate or incomparable language
classes.  Standard examples include
one-counter languages, deterministic context-free languages, context-free languages,
and recursive languages~\cite{HopcroftUllman}.
\end{remark}

\begin{proposition}[Multi-agent simulation overhead]
\label{prop:simulation-overhead}
The multi-agent simulations have the following costs.
\begin{enumerate}[leftmargin=1.5em]
\item For fixed $k$, one $\TMTn{k}$ round costs at most $k$ steps of
$\OTMn{k}$; in particular, one $\TMTn{1}$ round costs one $\OTMn{1}$ step.  A classical
Hopcroft--Ullman pushdown transition with longest replacement word $L$ first expands into
at most $L+3$ such rounds by Lemma~\ref{lem:classical-stack-normal-form}.
\item One deterministic Turing-machine step costs $O(1)$ steps of $\OTMn{2}$, and
$\OTMn{2}$ simulates $\OTMn{k}$ with polynomial overhead.
\end{enumerate}
The simulations use $O(t)$ transcript blocks for $t$ stack operations.
\end{proposition}

\begin{proof}
In the reverse simulation of
Theorem~\ref{thm:agent-stack-equivalence}, an all-stay $\TMTn{k}$ round is one
$\mathrm{idle}$ round.  Otherwise, each non-stay coordinate becomes one
$\mathrm{push}(i,b)$ or $\mathrm{pop}(i)$ round, so at most $k$ rounds are needed and
at most $k-1$ are internal and silent.  Encoding one stack symbol per block gives the
first bound and the storage cost.  The standard split-tape simulation then gives constant
overhead per Turing step.

The forward simulation in Theorem~\ref{thm:agent-stack-equivalence} maps each
orchestration round to one round of a fixed member of $\TMTn{k}$.  Simulating that
transducer on a Turing machine and returning through the two-stack construction gives
the polynomial bound.
\end{proof}

\section{Discussion}

\paragraph{Finite versus unbounded consulted memory}
With finite precision, bounded transcript visibility yields finite consulted state.  A
growing exact KV cache, writable store, suitable tool oracle, or transcript operation
that re-exposes hidden blocks supplies unbounded discrete memory
\cite{OrenEtAl2024,TiwariNalliDeshpande2026ToolOracles,PerezBarceloMarinkovic2021}.
Bounded read-only RAG supplies bounded external input~\cite{LewisEtAl2020RAG}.  In the
abstract controller, PopContext re-exposes the most recently hidden block by stack pop,
giving deterministic pushdown power with one pop-enabled transcript and Turing universality with two.

\paragraph{Chain-of-thought and repetitive loops}
Chain-of-thought expressivity~\cite{MerrillSabharwal2024CoT,FengEtAl2023CoT} scales with
decoding steps when each step consults a growing exact transcript and log-precision
arithmetic.
If each step consults only a fixed finite summary, Corollary~\ref{cor:cot-cycle} applies:
an infinite scratchpad run must eventually cycle.
Repetitive tool-and-thought loops under a token budget exhibit the same limitation.
Growing exact writable memory, unbounded agent population, and operations such as $\PopContext$
provide practical routes beyond this fixed-summary regime.

\paragraph{Formal-language interpretation}
Relative to Hahn~\cite{Hahn2020} and Merrill et al.~\cite{MerrillSabharwal2023,MerrillSabharwalSmith2022},
who upper-bound fixed architectures, we classify the transcript-management layer around a
fixed architecture.  We identify a deployed transcript operation with the pop action of
our normal form, connect that form to the classical presentation through
Lemma~\ref{lem:classical-stack-normal-form}, and prove that the specified
fixed-population monotone alternatives remain finite-state.
Pushdown models connect the transcript interface to deterministic
parsing~\cite{Knuth1965}, grammar-constrained or checkable
synthesis~\cite{AlurEtAl2013,KobayashiEtAl2025}, and instruction
languages~\cite{Kuhn2014,VeizagaEtAl2021}.

\paragraph{Context compaction}
For long-session compaction~\cite{RaeEtAl2019Compressive,DaiEtAl2025}, the relevant distinction is
whether hidden or discarded exact text can later be re-exposed.  A finite summary
remains finite-state.  PopContext-style re-exposure supplies stack memory.

\paragraph{Unbounded Spawn avoids finite-population collapse}
Unbounded \textbf{Spawn} creates an unbounded collection of local states even if each
agent uses fixed precision.
$\PopContext$ recovers hidden memory at a fixed population.  These operations add
distinct state resources.

\section{Conclusion}

The fixed-population monotone collapse is the paper's novel theoretical result.
Any finite collection of agents governed by the standard local-call, append, routing, and
visible-copy operations realizes exactly the finite-state transductions so cannot accurately
parse any general programming language.
This result applies directly to most deployments of Transformers in production systems,
and it gives a formal reason why synthesis and verification pipelines supply their
grammars and checkers from outside the learned
component~\cite{AlurEtAl2013,KobayashiEtAl2025}.

The bounded-transcript correspondence represents the same management layer as our
$\RTMTn{k}$, the pop-free case of the Transcript-Managed Transducer.  Channelwise
\PopContext admits the one remaining action and returns the full $\TMTn{k}$.
One pop-enabled transcript gives the deterministic context-free
languages, and two give the recursively enumerable languages.  The channels may reside in
one agent or in separate orchestrated agents.  For $k\ge2$, every effective universal
augmentation accepts the same language class, $\RE$.  One operation extends the
transcript-management vocabulary, while the Transformer, its weights, its token protocol,
and the amount of retained text stay as they are.  Deployed single- and multi-agent
systems therefore admit a transcript-managed reading, which opens a route 
to strengthening existing deployments by adding pop locally in the transcript
manager rather than changing the model.

\bibliographystyle{IEEEtran}
\bibliography{main}

@article{PerezBarceloMarinkovic2021,
  author = {Jorge P{\'e}rez and Pablo Barcel{\'o} and Javier Marinkovic},
  title = {Attention is {Turing}-Complete},
  journal = {Journal of Machine Learning Research},
  volume = {22},
  number = {75},
  pages = {1--35},
  year = {2021},
  comment = {Summary: Proves Turing completeness for attention models using arbitrary-precision internal representations. Relevance: Contrasts that unbounded numeric-state regime with this paper's fixed-precision model. Access: https://jmlr.org/papers/v22/20-302.html}
}

@inproceedings{VaswaniEtAl2017,
  author = {Ashish Vaswani and Noam Shazeer and Niki Parmar and Jakob Uszkoreit and Llion Jones and Aidan N. Gomez and {\L}ukasz Kaiser and Illia Polosukhin},
  title = {Attention Is All You Need},
  booktitle = {Advances in Neural Information Processing Systems 30},
  year = {2017},
  eprint = {1706.03762},
  archivePrefix = {arXiv},
  primaryClass = {cs.CL},
  comment = {Summary: Introduces the Transformer architecture based on self-attention. Relevance: Provides the baseline architecture whose transcript-management capabilities are studied here. Access: https://arxiv.org/abs/1706.03762}
}

@misc{RadfordEtAl2019GPT2,
  author = {Alec Radford and Jeffrey Wu and Rewon Child and David Luan and Dario Amodei and Ilya Sutskever},
  title = {Language Models are Unsupervised Multitask Learners},
  year = {2019},
  howpublished = {OpenAI technical report},
  url = {https://cdn.openai.com/better-language-models/language_models_are_unsupervised_multitask_learners.pdf},
  comment = {Summary: Introduces GPT-2 with a 1024-token context. Relevance: Provides a concrete fixed-context scale for the physical visibility radius.}
}

@inproceedings{BrownEtAl2020GPT3,
  author = {Tom B. Brown and Benjamin Mann and Nick Ryder and Melanie Subbiah and Jared Kaplan and Prafulla Dhariwal and Arvind Neelakantan and Pranav Shyam and Girish Sastry and Amanda Askell and Sandhini Agarwal and Ariel Herbert-Voss and Gretchen Krueger and Tom Henighan and Rewon Child and Aditya Ramesh and Daniel M. Ziegler and Jeffrey Wu and Clemens Winter and Christopher Hesse and Mark Chen and Eric Sigler and Mateusz Litwin and Scott Gray and Benjamin Chess and Jack Clark and Christopher Berner and Sam McCandlish and Alec Radford and Ilya Sutskever and Dario Amodei},
  title = {Language Models are Few-Shot Learners},
  booktitle = {Advances in Neural Information Processing Systems},
  volume = {33},
  pages = {1877--1901},
  year = {2020},
  eprint = {2005.14165},
  archivePrefix = {arXiv},
  primaryClass = {cs.CL},
  comment = {Summary: Introduces GPT-3 and reports training on a 2048-token context window. Relevance: Provides a second concrete fixed-context scale for the physical visibility radius.}
}

@misc{DehghaniEtAl2018,
  author = {Mostafa Dehghani and Stephan Gouws and Oriol Vinyals and Jakob Uszkoreit and {\L}ukasz Kaiser},
  title = {Universal Transformers},
  year = {2018},
  note = {arXiv:1807.03819},
  eprint = {1807.03819},
  archivePrefix = {arXiv},
  primaryClass = {cs.LG},
  url = {https://arxiv.org/abs/1807.03819},
  comment = {Summary: Extends Transformers with recurrent depth and adaptive computation. Relevance: Illustrates a recurrent architectural route to stronger computational behavior.}
}

@inproceedings{OrenEtAl2024,
  author = {Matanel Oren and Michael Hassid and Nir Yarden and Yossi Adi and Roy Schwartz},
  title = {Transformers Are Multi-State {RNN}s},
  booktitle = {Proceedings of the 2024 Conference on Empirical Methods in Natural Language Processing},
  pages = {18724--18741},
  year = {2024},
  doi = {https://doi.org/10.18653/v1/2024.emnlp-main.1043},
  comment = {Summary: Characterizes decoder-only Transformers as multi-state recurrent systems whose state grows with the key-value cache. Relevance: Supports the distinction between fixed consulted state and unbounded retained cache state.}
}

@book{HopcroftUllman,
  author = {John E. Hopcroft and Rajeev Motwani and Jeffrey D. Ullman},
  title = {Introduction to Automata Theory, Languages, and Computation},
  publisher = {Pearson},
  edition = {3},
  year = {2007},
  comment = {Summary: Develops the classical hierarchy of finite automata, pushdown automata, Turing machines, and randomized computation. Relevance: Supplies the automata-theoretic definitions, simulations, and random-tape construction used throughout this paper.}
}

@article{Atig2012MultiPushdown,
  author = {Mohamed Faouzi Atig},
  title = {Model-Checking of Ordered Multi-Pushdown Automata},
  journal = {Logical Methods in Computer Science},
  volume = {8},
  number = {3},
  pages = {1--31},
  year = {2012},
  doi = {https://doi.org/10.2168/LMCS-8(3:20)2012},
  eprint = {1209.1916},
  archivePrefix = {arXiv},
  primaryClass = {cs.LO},
  comment = {Summary: Recalls the standard multi-pushdown automaton with finitely many LIFO stacks before imposing an ordered-pop restriction. Relevance: Supplies a precise prior multi-stack definition against which our deterministic transducer normal form is compared.}
}

@book{Minsky1967,
  author = {Marvin L. Minsky},
  title = {Computation: Finite and Infinite Machines},
  publisher = {Prentice-Hall},
  year = {1967},
  isbn = {0131655639},
  comment = {Summary: Presents foundational models of finite and universal computation, including machine simulations. Relevance: Supports the classical two-stack and Turing-machine equivalences used in the universality argument. Access: https://www.openphilanthropy.org/files/Focus_Areas/AI/Minsky_1967.pdf}
}

@article{Hahn2020,
  author = {Michael Hahn},
  title = {Theoretical Limitations of Self-Attention in Neural Sequence Models},
  journal = {Transactions of the Association for Computational Linguistics},
  volume = {8},
  pages = {156--171},
  year = {2020},
  doi = {https://doi.org/10.1162/tacl_a_00306},
  comment = {Summary: Establishes formal-language limitations of self-attention under bounded architectural resources. Relevance: Positions this paper's transcript primitive relative to known Transformer expressivity barriers.}
}

@article{MerrillSabharwal2023,
  author = {William Merrill and Ashish Sabharwal},
  title = {The Parallelism Tradeoff: Limitations of Log-Precision Transformers},
  journal = {Transactions of the Association for Computational Linguistics},
  volume = {11},
  pages = {531--545},
  year = {2023},
  doi = {https://doi.org/10.1162/tacl_a_00562},
  comment = {Summary: Relates log-precision Transformers to uniform constant-depth threshold circuits. Relevance: Provides a circuit-complexity upper bound for comparison with the fixed-precision transcript model.}
}

@inproceedings{BhattamishraEtAl2020,
  author = {Satwik Bhattamishra and Kabir Ahuja and Navin Goyal},
  title = {On the Ability and Limitations of Transformers to Recognize Formal Languages},
  booktitle = {Proceedings of the 2020 Conference on Empirical Methods in Natural Language Processing},
  pages = {7096--7116},
  year = {2020},
  doi = {https://doi.org/10.18653/v1/2020.emnlp-main.576},
  comment = {Summary: Constructs Transformers for restricted counter languages and reports positive finite-range formal-language generalization results, including for \(a^n b^n c^n\), using growing full-prefix attention. Relevance: Provides the direct contrasting experiment that separates approximate precision- and context-dependent counter computation from the paper's exact fixed-visibility machine classification.}
}

@inproceedings{DeletangEtAl2023,
  author = {Gr{\'e}goire Del{\'e}tang and Anian Ruoss and Jordi Grau-Moya and Tim Genewein and Li Kevin Wenliang and Elliot Catt and Chris Cundy and Marcus Hutter and Shane Legg and Joel Veness and Pedro A. Ortega},
  title = {Neural Networks and the {Chomsky} Hierarchy},
  booktitle = {International Conference on Learning Representations},
  year = {2023},
  eprint = {2207.02098},
  archivePrefix = {arXiv},
  primaryClass = {cs.LG},
  comment = {Summary: Empirically compares standard and structured-memory neural networks on transduction tasks spanning the Chomsky hierarchy. Relevance: Supplies the closest broad empirical precedent for relating learned stack or tape memory to formal-language length generalization. Access: https://arxiv.org/abs/2207.02098}
}

@article{StroblEtAl2024Survey,
  author = {Lena Strobl and William Merrill and Gail Weiss and David Chiang and Dana Angluin},
  title = {What Formal Languages Can Transformers Express? {A} Survey},
  journal = {Transactions of the Association for Computational Linguistics},
  volume = {12},
  pages = {543--561},
  year = {2024},
  doi = {https://doi.org/10.1162/tacl_a_00663},
  comment = {Summary: Surveys Transformer formal-language expressivity under differing architectural, precision, positional, and uniformity assumptions. Relevance: Provides the canonical synthesis for the upper-bound literature against which the transcript-management model is positioned.}
}

@inproceedings{MerrillSabharwal2024CoT,
  author = {William Merrill and Ashish Sabharwal},
  title = {The Expressive Power of Transformers with Chain of Thought},
  booktitle = {International Conference on Learning Representations},
  year = {2024},
  comment = {Summary: Characterizes how polynomially many chain-of-thought steps increase Transformer expressive power. Relevance: Contrasts growing exact scratchpads with fixed finite consulted state. Access: https://openreview.net/forum?id=CDmerQ37Zs}
}

@inproceedings{FengEtAl2023CoT,
  author = {Guhao Feng and Bohang Zhang and Yuntian Gu and Haotian Ye and Di He and Liwei Wang},
  title = {Towards Revealing the Mystery behind Chain of Thought: A Theoretical Perspective},
  booktitle = {Advances in Neural Information Processing Systems},
  volume = {36},
  year = {2023},
  comment = {Summary: Analyzes chain of thought as added effective depth for arithmetic and dynamic-programming tasks. Relevance: Supports the discussion of transcript growth as a computational resource. Access: https://proceedings.neurips.cc/paper_files/paper/2023/hash/dfc310e81992d2e4cedc09ac47eff13e-Abstract.html}
}

@article{MerrillSabharwalSmith2022,
  author = {William Merrill and Ashish Sabharwal and Noah A. Smith},
  title = {Saturated Transformers are Constant-Depth Threshold Circuits},
  journal = {Transactions of the Association for Computational Linguistics},
  volume = {10},
  pages = {843--856},
  year = {2022},
  doi = {https://doi.org/10.1162/tacl_a_00493},
  comment = {Summary: Shows that saturated Transformers can be simulated by constant-depth threshold circuits. Relevance: Gives an earlier restricted-attention upper bound against which this model is compared.}
}

@inproceedings{StearnsHartmanisLewis1965,
  author = {R. E. Stearns and J. Hartmanis and Lewis, II, P. M.},
  title = {Hierarchies of Memory Limited Computations},
  booktitle = {6th Annual Symposium on Switching Circuit Theory and Logical Design},
  pages = {179--190},
  year = {1965},
  doi = {https://doi.org/10.1109/FOCS.1965.11},
  comment = {Summary: Establishes strict hierarchies for computations with bounded memory. Relevance: Provides classical context for the paper's memory-sensitive expressivity thresholds.}
}

@inproceedings{RaeEtAl2019Compressive,
  author = {Jack W. Rae and Anna Potapenko and Siddhant M. Jayakumar and Chloe Hillier and Timothy P. Lillicrap},
  title = {Compressive Transformers for Long-Range Sequence Modelling},
  booktitle = {International Conference on Learning Representations},
  year = {2020},
  eprint = {1911.05507},
  archivePrefix = {arXiv},
  primaryClass = {cs.LG},
  comment = {Summary: Compresses past activations into a secondary memory instead of discarding them. Relevance: Canonical lossy finite-memory compaction for Transformer long-context models. Access: https://openreview.net/forum?id=SylKikSYDH}
}

@inproceedings{DaiEtAl2019TransformerXL,
  author = {Zihang Dai and Zhilin Yang and Yiming Yang and Jaime Carbonell and Quoc V. Le and Ruslan Salakhutdinov},
  title = {{Transformer}-{XL}: Attentive Language Models beyond a Fixed-Length Context},
  booktitle = {Proceedings of the 57th Annual Meeting of the Association for Computational Linguistics},
  pages = {2978--2988},
  year = {2019},
  doi = {https://doi.org/10.18653/v1/P19-1285},
  comment = {Summary: Introduces Transformer-XL with segment-level recurrence over cached hidden states and relative positional encodings. Relevance: The base long-context recurrence model on which later layerwise-memory and compressive variants build.}
}

@inproceedings{RaeRazavi2020,
  author = {Jack W. Rae and Ali Razavi},
  title = {Do Transformers Need Deep Long-Range Memory?},
  booktitle = {Proceedings of the 58th Annual Meeting of the Association for Computational Linguistics},
  pages = {7524--7529},
  year = {2020},
  doi = {https://doi.org/10.18653/v1/2020.acl-main.672},
  comment = {Summary: Follow-up to Transformer-XL showing that long-range memory need not be placed in every layer. Relevance: Separates layerwise long-range memory allocation from finite-summary compaction.}
}

@inproceedings{DaiEtAl2025,
  author = {Yuhong Dai and Jianxun Lian and Yitian Huang and Wei Zhang and Mingyang Zhou and Mingqi Wu and Xing Xie and Hao Liao},
  title = {Pretraining Context Compressor for Large Language Models with Embedding-Based Memory},
  booktitle = {Proceedings of the 63rd Annual Meeting of the Association for Computational Linguistics (Volume 1: Long Papers)},
  pages = {28715--28732},
  year = {2025},
  doi = {https://doi.org/10.18653/v1/2025.acl-long.1394},
  comment = {Summary: Pretrains an embedding-based context compressor for long-context language models. Relevance: Provides a modern example of finite-summary context management.}
}

@inproceedings{LewisEtAl2020RAG,
  author = {Patrick Lewis and Ethan Perez and Aleksandra Piktus and Fabio Petroni and Vladimir Karpukhin and Naman Goyal and Heinrich K{\"u}ttler and Mike Lewis and Wen-tau Yih and Tim Rockt{\"a}schel and Sebastian Riedel and Douwe Kiela},
  title = {Retrieval-Augmented Generation for Knowledge-Intensive {NLP} Tasks},
  booktitle = {Advances in Neural Information Processing Systems 33},
  year = {2020},
  eprint = {2005.11401},
  archivePrefix = {arXiv},
  primaryClass = {cs.CL},
  comment = {Summary: Introduces retrieval-augmented generation using an external nonparametric document memory. Relevance: Serves as the canonical read-only external-store comparison. Access: https://arxiv.org/abs/2005.11401}
}

@article{Knuth1965,
  author = {Donald E. Knuth},
  title = {On the Translation of Languages from Left to Right},
  journal = {Information and Control},
  volume = {8},
  number = {6},
  pages = {607--639},
  year = {1965},
  doi = {https://doi.org/10.1016/S0019-9958(65)90426-2},
  comment = {Summary: Develops LR parsing for deterministic context-free languages. Relevance: Supports deterministic parsing as an application of the pushdown layer.}
}

@inproceedings{AlurMadhusudan2004,
  author = {Rajeev Alur and P. Madhusudan},
  title = {Visibly Pushdown Languages},
  booktitle = {Proceedings of the Thirty-Sixth Annual {ACM} Symposium on Theory of Computing},
  pages = {202--211},
  year = {2004},
  doi = {https://doi.org/10.1145/1007352.1007390},
  comment = {Summary: Introduces visibly pushdown languages, where the input symbol determines push/pop. Relevance: Supplies a classical intermediate stack restriction between unrestricted pushdown and finite-state control.}
}

@inproceedings{AlurEtAl2013,
  author = {Rajeev Alur and Rastislav Bodik and Garvit Juniwal and Milo M. K. Martin and Mukund Raghothaman and Sanjit A. Seshia and Rishabh Singh and Armando Solar-Lezama and Emina Torlak and Abhishek Udupa},
  title = {Syntax-Guided Synthesis},
  booktitle = {Proceedings of the 2013 Formal Methods in Computer-Aided Design},
  pages = {1--17},
  year = {2013},
  comment = {Summary: Introduces syntax-guided synthesis, where grammars constrain the search for programs. Relevance: Motivates grammar-constrained and checkable synthesis as a target application. Access: https://people.eecs.berkeley.edu/~sseshia/pubs/b2hd-alur-fmcad13.html}
}

@article{Kuhn2014,
  author = {Tobias Kuhn},
  title = {A Survey and Classification of Controlled Natural Languages},
  journal = {Computational Linguistics},
  volume = {40},
  number = {1},
  pages = {121--170},
  year = {2014},
  doi = {https://doi.org/10.1162/COLI_a_00168},
  comment = {Summary: Surveys and classifies controlled natural languages by syntax, semantics, and use. Relevance: Supports controlled instruction languages as a sample-first compilation target.}
}

@article{VeizagaEtAl2021,
  author = {Alvaro Veizaga and Mauricio Alf{\'e}rez and Damiano Torre and Mehrdad Sabetzadeh and Lionel Briand},
  title = {On Systematically Building a Controlled Natural Language for Functional Requirements},
  journal = {Empirical Software Engineering},
  volume = {26},
  number = {5},
  pages = {90},
  year = {2021},
  doi = {https://doi.org/10.1007/s10664-021-09956-6},
  comment = {Summary: Presents a systematic method for constructing controlled natural languages for requirements. Relevance: Provides an applied model for designing constrained instruction languages.}
}

@article{KobayashiEtAl2025,
  author = {Naoki Kobayashi and Taro Sekiyama and Issei Sato and Hiroshi Unno},
  title = {Towards Neural-Network-Guided Program Synthesis and Verification},
  journal = {Formal Methods in System Design},
  year = {2025},
  doi = {https://doi.org/10.1007/s10703-024-00468-9},
  comment = {Summary: Surveys and advances neural guidance for program synthesis and verification. Relevance: Connects learned controllers with formally checkable program construction.}
}

@inproceedings{JoulinMikolov2015,
  author = {Armand Joulin and Tomas Mikolov},
  title = {Inferring Algorithmic Patterns with Stack-Augmented Recurrent Nets},
  booktitle = {Advances in Neural Information Processing Systems 28},
  year = {2015},
  eprint = {1503.01007},
  archivePrefix = {arXiv},
  primaryClass = {cs.NE},
  comment = {Summary: Uses soft push, pop, and no-op actions in a superposition stack controlled by a recurrent network. Relevance: Provides an independent differentiable-stack recurrence distinct from the continuous-strength update used here. Access: https://arxiv.org/abs/1503.01007}
}

@inproceedings{GrefenstetteEtAl2015,
  author = {Edward Grefenstette and Karl Moritz Hermann and Mustafa Suleyman and Phil Blunsom},
  title = {Learning to Transduce with Unbounded Memory},
  booktitle = {Advances in Neural Information Processing Systems 28},
  year = {2015},
  eprint = {1506.02516},
  archivePrefix = {arXiv},
  primaryClass = {cs.CL},
  comment = {Summary: Introduces differentiable stacks, queues, and deques using continuous strengths, newest-first deletion, and a unit-capacity read. Relevance: Supplies the strength update and top-read recurrence adopted by the training surrogate. Access: https://arxiv.org/abs/1506.02516}
}

@inproceedings{ChungSiegelmann2021,
  author = {Stephen Chung and Hava T. Siegelmann},
  title = {Turing Completeness of Bounded-Precision Recurrent Neural Networks},
  booktitle = {Advances in Neural Information Processing Systems 34},
  year = {2021},
  pages = {28431--28441},
  comment = {Summary: Proves that a bounded-precision RNN with two dynamically growing stack modules can simulate a universal Turing machine and extends the construction to stack-augmented RNNs, including Grefenstette-style neural stacks. Relevance: Locates differentiable-stack universality in the growth of two exact stack memories rather than continuous strengths alone. Access: https://proceedings.neurips.cc/paper/2021/hash/ef452c63f81d0105dd4486f775adec81-Abstract.html}
}

@inproceedings{DuSellChiang2020,
  author = {Brian DuSell and David Chiang},
  title = {Learning Context-free Languages with Nondeterministic Stack {RNN}s},
  booktitle = {Proceedings of the 24th Conference on Computational Natural Language Learning},
  year = {2020},
  pages = {507--519},

  doi = {https://doi.org/10.18653/v1/2020.conll-1.41},
  comment = {Summary: Uses nondeterministic stack recurrent networks to learn context-free languages. Relevance: Provides a learned pushdown model for comparison with the deterministic transcript controller.}
}

@inproceedings{DuSellChiang2024,
  author = {Brian DuSell and David Chiang},
  title = {Stack Attention: Improving the Ability of Transformers to Model Hierarchical Patterns},
  booktitle = {International Conference on Learning Representations},
  year = {2024},
  comment = {Summary: Adapts superposition and nondeterministic differentiable stacks into Transformer attention operators. Relevance: Provides the closest prior Transformer integration of differentiable stack computation. Access: https://openreview.net/forum?id=XVhm3X8Fum}
}

@inproceedings{MurtyEtAl2023,
  author = {Shikhar Murty and Pratyusha Sharma and Jacob Andreas and Christopher D. Manning},
  title = {Pushdown Layers: Encoding Recursive Structure in Transformer Language Models},
  booktitle = {Proceedings of the 2023 Conference on Empirical Methods in Natural Language Processing},
  year = {2023},
  pages = {3233--3247},

  doi = {https://doi.org/10.18653/v1/2023.emnlp-main.195},
  comment = {Summary: Uses a learned stack tape to modulate Transformer attention over incrementally parsed prefixes. Relevance: Provides a distinct Transformer architecture with explicit soft pushdown structure.}
}

@inproceedings{PressEtAl2021ALiBi,
  author = {Ofir Press and Noah A. Smith and Mike Lewis},
  title = {Train Short, Test Long: Attention with Linear Biases Enables Input Length Extrapolation},
  booktitle = {International Conference on Learning Representations},
  year = {2022},
  eprint = {2108.12409},
  archivePrefix = {arXiv},
  primaryClass = {cs.CL},
  comment = {Summary: Introduces linear attention biases that enable extrapolation beyond training lengths. Relevance: Supports the claim that length generalization depends on positional encoding. Access: https://openreview.net/forum?id=R8sQPpGCv0}
}

@inproceedings{KazemnejadEtAl2023PE,
  author = {Amirhossein Kazemnejad and Inkit Padhi and Karthikeyan Natesan Ramamurthy and Payel Das and Siva Reddy},
  title = {The Impact of Positional Encoding on Length Generalization in Transformers},
  booktitle = {Advances in Neural Information Processing Systems},
  volume = {36},
  year = {2023},
  comment = {Summary: Empirically and theoretically studies how positional encodings affect Transformer length generalization. Relevance: Establishes positional encoding as a separate constraint on usable context. Access: https://proceedings.neurips.cc/paper_files/paper/2023/hash/4e85362c02172c0c6567ce593122d31c-Abstract-Conference.html}
}

@article{SuEtAl2024RoFormer,
  author = {Jianlin Su and Murtadha Ahmed and Yu Lu and Shengfeng Pan and Wen Bo and Yunfeng Liu},
  title = {{RoFormer}: Enhanced Transformer with Rotary Position Embedding},
  journal = {Neurocomputing},
  volume = {568},
  pages = {127063},
  year = {2024},
  eprint = {2104.09864},
  archivePrefix = {arXiv},
  primaryClass = {cs.CL},
  doi = {https://doi.org/10.1016/j.neucom.2023.127063},
  comment = {Summary: Introduces rotary position embedding for encoding relative position through feature rotations. Relevance: Defines the positional mechanism whose finite-precision limits are discussed.}
}

@inproceedings{MenEtAl2024RoPEBase,
  author = {Mingyu Xu and Xin Men and Bingning Wang and Qingyu Zhang and Hongyu Lin and Yaojie Lu and Xianpei Han and Weipeng Chen},
  title = {Base of {RoPE} Bounds Context Length},
  booktitle = {Advances in Neural Information Processing Systems},
  volume = {37},
  year = {2024},
  comment = {Summary: Shows that the RoPE base constrains effective context length and proposes scaling guidance. Relevance: Supports the finite-context consequences of rotary-position parameterization. Access: https://proceedings.neurips.cc/paper_files/paper/2024/hash/9f12dd32d552f3ad9eaa0e9dfec291be-Abstract-Conference.html}
}

@misc{Liu2026RoPEPhase,
  author = {Feilong Liu},
  title = {Rotary Positional Embeddings as Phase Modulation: Theoretical Bounds on the {RoPE} Base for Long-Context Transformers},
  year = {2026},
  eprint = {2602.10959},
  archivePrefix = {arXiv},
  primaryClass = {cs.LG},
  url = {https://arxiv.org/abs/2602.10959},
  comment = {Summary: Derives precision- and depth-dependent bounds on the RoPE base using a phase-modulation view. Relevance: Supports the phase-erasure limitation under fixed floating-point arithmetic.}
}

@inproceedings{WeissGoldbergYahav2018,
  author = {Gail Weiss and Yoav Goldberg and Eran Yahav},
  title = {Extracting Automata from Recurrent Neural Networks Using Queries and Counterexamples},
  booktitle = {Proceedings of the 35th International Conference on Machine Learning},
  year = {2018},
  comment = {Summary: Given a fixed, already-trained recurrent network, extracts a finite automaton afterward using active queries and counterexamples. Relevance: Motivates post-training extraction and certification of learned finite-state controllers without treating extraction as part of network training. Access: https://proceedings.mlr.press/v80/weiss18a.html}
}

@inproceedings{LiWang2026EfficientTM,
  author = {Qian Li and Yuyi Wang},
  title = {Efficient {Turing} Machine Simulation with {Transformers}},
  booktitle = {International Conference on Learning Representations},
  year = {2026},
  eprint = {2512.00003},
  archivePrefix = {arXiv},
  primaryClass = {cs.LG},
  comment = {Summary: Simulates space-bounded Turing machines with constant-bit Transformers using an $O(s(n))$ context and $O(s(n)^c)$ chain-of-thought steps per simulated step. Relevance: Contrasts efficient computation with growing exact context and CoT against this paper's fixed-visibility threshold. Access: https://arxiv.org/abs/2512.00003}
}

@inproceedings{YaoEtAl2023ReAct,
  author = {Shunyu Yao and Jeffrey Zhao and Dian Yu and Nan Du and Izhak Shafran and Karthik Narasimhan and Yuan Cao},
  title = {{ReAct}: Synergizing Reasoning and Acting in Language Models},
  booktitle = {International Conference on Learning Representations},
  year = {2023},
  comment = {Summary: Interleaves language-model reasoning traces with actions in an external environment. Relevance: Provides a representative tool-using scratchpad protocol. Access: https://openreview.net/forum?id=WE_vluYUL-X}
}

@inproceedings{WuEtAl2024AutoGen,
  author = {Qingyun Wu and Gagan Bansal and Jieyu Zhang and Yiran Wu and Beibin Li and Erkang Zhu and Li Jiang and Xiaoyun Zhang and Shaokun Zhang and Jiale Liu and Ahmed Hassan Awadallah and Ryen W. White and Doug Burger and Chi Wang},
  title = {{AutoGen}: Enabling Next-Gen {LLM} Applications via Multi-Agent Conversation},
  booktitle = {First Conference on Language Modeling},
  year = {2024},
  comment = {Summary: Presents a framework for applications built from conversing, tool-using language-model agents. Relevance: Supplies a practical multi-agent handoff comparison. Access: https://openreview.net/forum?id=BAakY1hNKS}
}

@misc{RizviMartelEtAl2025MultiAgent,
  author = {Michael Rizvi-Martel and Satwik Bhattamishra and Neil Rathi and Guillaume Rabusseau and Michael Hahn},
  title = {Benefits and Limitations of Communication in Multi-Agent Reasoning},
  year = {2025},
  eprint = {2510.13903},
  archivePrefix = {arXiv},
  primaryClass = {cs.LG},
  url = {https://arxiv.org/abs/2510.13903},
  comment = {Summary: Derives tradeoffs among agent count, communication, and depth for multi-agent reasoning tasks. Relevance: Positions this paper's fixed-population transcript hierarchy relative to communication-based gains.}
}

@inproceedings{TiwariNalliDeshpande2026ToolOracles,
  author = {Utkarsh Tiwari and Sai Soumya Nalli and Amit Deshpande},
  title = {Modeling Tool Use in Transformers via Computation Oracles},
  booktitle = {Latent and Implicit Thinking Workshop at {ICLR}},
  year = {2026},
  comment = {Summary: Models Transformer tool calls as computation-oracle access and derives expressivity gains. Relevance: Identifies oracle access as a separate route beyond fixed finite consulted state. Access: https://openreview.net/forum?id=c5LfhgWztB}
}

@misc{HouEtAl2026AgenticLoops,
  author = {Xinyi Hou and Shenao Wang and Yanjie Zhao and Haoyu Wang},
  title = {When Agents Do Not Stop: Uncovering Infinite Agentic Loops in {LLM} Agents},
  year = {2026},
  eprint = {2607.01641},
  archivePrefix = {arXiv},
  primaryClass = {cs.SE},
  comment = {Summary: Audits real-world LLM-agent repositories and manually confirms failures involving unbounded repetition of model calls, tools, workflow transitions, or handoffs. Relevance: Supplies empirical evidence for the repetitive tool-and-thought cycling that our finite-state result characterizes formally.}
}

@article{HaoAngluinFrank2022,
  author = {Yiding Hao and Dana Angluin and Robert Frank},
  title = {Formal Language Recognition by Hard Attention Transformers: Perspectives from Circuit Complexity},
  journal = {Transactions of the Association for Computational Linguistics},
  volume = {10},
  pages = {800--810},
  year = {2022},
  doi = {10.1162/tacl_a_00490},
  comment = {Summary: Places unique-hard and generalized unique-hard attention within AC0, and shows that averaging-hard attention recognizes the non-AC0 languages MAJORITY and DYCK-1. Relevance: Distinguishes hard-attention fixed-depth models from saturated or soft-attention settings in the upper-bound taxonomy.}
}

@inproceedings{BhattamishraPatelGoyal2020,
  author = {Satwik Bhattamishra and Arkil Patel and Navin Goyal},
  title = {On the Computational Power of Transformers and Its Implications in Sequence Modeling},
  booktitle = {Proceedings of the 24th Conference on Computational Natural Language Learning (CoNLL)},
  pages = {455--475},
  year = {2020},
  doi = {10.18653/v1/2020.conll-1.37},
  comment = {Summary: Analyzes the computational power of Transformers under model-specific assumptions, including Turing completeness with unbounded precision. Relevance: Broadens the universality discussion beyond a single arbitrary-precision construction.}
}

@inproceedings{BhattamishraSimplicityBias2023,
  author = {Satwik Bhattamishra and Arkil Patel and Varun Kanade and Phil Blunsom},
  title = {Simplicity Bias in Transformers and their Ability to Learn Sparse {Boolean} Functions},
  booktitle = {Proceedings of the 61st Annual Meeting of the Association for Computational Linguistics (Volume 1: Long Papers)},
  pages = {5767--5791},
  year = {2023},
  doi = {10.18653/v1/2023.acl-long.317},
  comment = {Summary: Shows Transformers are biased toward low-sensitivity (simple) functions, affecting which formal patterns they generalize. Relevance: Indicates that optimization bias, not only memory structure, shapes the reported partial length generalization.}
}

@inproceedings{ZhangStackTrans2025,
  author = {Kechi Zhang and Ge Li and Jia Li and Huangzhao Zhang and Yihong Dong and Jia Li and Jingjing Xu and Zhi Jin},
  title = {Recursive Transformer: Boosting Reasoning Ability with State Stack},
  booktitle = {Advances in Neural Information Processing Systems 38},
  year = {2025},
  comment = {Summary: Inserts differentiable hidden-state stacks between Transformer layers, with end-to-end learned push and pop, and evaluates on Chomsky-hierarchy and natural-language reasoning benchmarks. Relevance: A depth-wise stack-augmented Transformer, contrasting with the transcript-wise tagged stacks used here. Access: https://openreview.net/forum?id=2bbDg587uh}
}

\end{document}